\documentclass[10pt,letterpaper]{article}

\usepackage[left=1.05in,right=1.05in,top=1.2in,bottom=0.75in,headheight=15pt,headsep=40.8pt,footskip=28.3pt]{geometry}

\usepackage[T1]{fontenc}
\usepackage{times}
\makeatletter
\renewcommand\normalsize{%
  \@setfontsize\normalsize{10}{11}%
  \abovedisplayskip 8pt plus 2pt minus 2pt
  \abovedisplayshortskip 0pt plus 2pt
  \belowdisplayshortskip 5pt plus 2pt minus 2pt
  \belowdisplayskip \abovedisplayskip
  \let\@listi\@listI}
\makeatother
\normalsize

\usepackage{amsmath,amsfonts,bm}

\def\eqref#1{equation~\ref{#1}}
\def\1{\bm{1}}

\DeclareMathAlphabet{\mathsfit}{\encodingdefault}{\sfdefault}{m}{sl}
\SetMathAlphabet{\mathsfit}{bold}{\encodingdefault}{\sfdefault}{bx}{n}

\newcommand{\Cov}{\mathrm{Cov}}
\usepackage{amsmath,amssymb,amsthm}
\usepackage{mathtools}

\usepackage{booktabs}
\usepackage{graphicx}

\usepackage{enumitem}
\usepackage{xcolor}
\usepackage{array,tabularx,colortbl}
\usepackage{microtype}
\usepackage{titlesec}
\usepackage[indent=0pt,skip=6pt plus 1pt minus 1pt]{parskip}
\usepackage[font=normalsize,labelfont=normalfont,labelsep=colon,skip=6pt]{caption}
\usepackage{fancyhdr}
\definecolor{paperblue}{HTML}{245580}
\definecolor{cotselink}{HTML}{1B2735}
\definecolor{cotselblue}{HTML}{0A4BD6}
\fancypagestyle{opening}{%
  \fancyhf{}
  \fancyhead[L]{\sffamily\bfseries\fontsize{15}{16}\selectfont\textcolor{cotselink}{CoT}\textcolor{cotselblue}{sel}}
  \fancyfoot[C]{\normalfont\fontsize{10}{11}\selectfont\thepage}
  }

\usepackage{natbib}
\setcitestyle{authoryear,square,aysep={,},yysep={,},notesep={, }}

\usepackage[colorlinks=true,linkcolor=paperblue,citecolor=paperblue,urlcolor=paperblue,
  pdftitle={Correctness Conditioning as Selection in Chain-of-Thought Faithfulness Evaluation},
  pdfauthor={Yueer Pan; Yuexi Pan},
  pdfsubject={Correctness-conditioned estimands for chain-of-thought faithfulness evaluation}
]{hyperref}
\usepackage{url}
\titleformat{\section}{\normalfont\fontsize{12}{13}\selectfont\scshape}{\thesection}{1em}{}
\titleformat{\subsection}{\normalfont\fontsize{10}{11}\selectfont\scshape}{\thesubsection}{1em}{}
\titleformat{\subsubsection}{\normalfont\normalsize\bfseries}{\thesubsubsection}{1em}{}
\titlespacing*{\section}{0pt}{16pt plus 2pt minus 2pt}{9pt}
\titlespacing*{\subsection}{0pt}{13pt plus 2pt minus 2pt}{8pt}
\titlespacing*{\subsubsection}{0pt}{10pt plus 2pt minus 2pt}{5pt}
\renewenvironment{abstract}{%
  \par\vspace{6pt}\noindent
  {\color{paperblue}\sffamily\bfseries\fontsize{10}{11}\selectfont ABSTRACT}\hspace{0.8em}%
  {\color{paperblue!30}\leaders\hrule height 0.3pt\hfill\kern0pt}\par\vspace{3pt}%
  \normalfont\fontsize{9}{10}\selectfont\setlength{\parindent}{0pt}%
}{\par}

\newtheorem{definition}{Definition}
\newtheorem{proposition}{Proposition}
\newtheorem{corollary}{Corollary}

\newcommand{\Fpop}{F_{\mathrm{pop}}}
\newcommand{\Fsel}{F_{\mathrm{sel}}}
\renewcommand{\Cov}{\mathrm{Cov}}

\title{Correctness Conditioning as Selection in\\Chain-of-Thought Faithfulness Evaluation}

\author{\textbf{Yueer Pan}\textsuperscript{1} \quad \textbf{Yuexi Pan}\textsuperscript{2}\\
\textnormal{\textsuperscript{1}Zhejiang University}\\
\textnormal{\textsuperscript{2}University of Science and Technology of China}\\
{\fontsize{10}{12}\selectfont\texttt{panjoel@zju.edu.cn \quad panyuexi@mail.ustc.edu.cn}}\\[4pt]
{\fontsize{10}{12}\selectfont\sffamily
\href{https://ccs-cotfe.github.io/CCS-CoTFE/}{\color{black}\raisebox{\dimexpr .5\fontcharht\font`H-.5\height\relax}{\includegraphics[width=1em,height=1em]{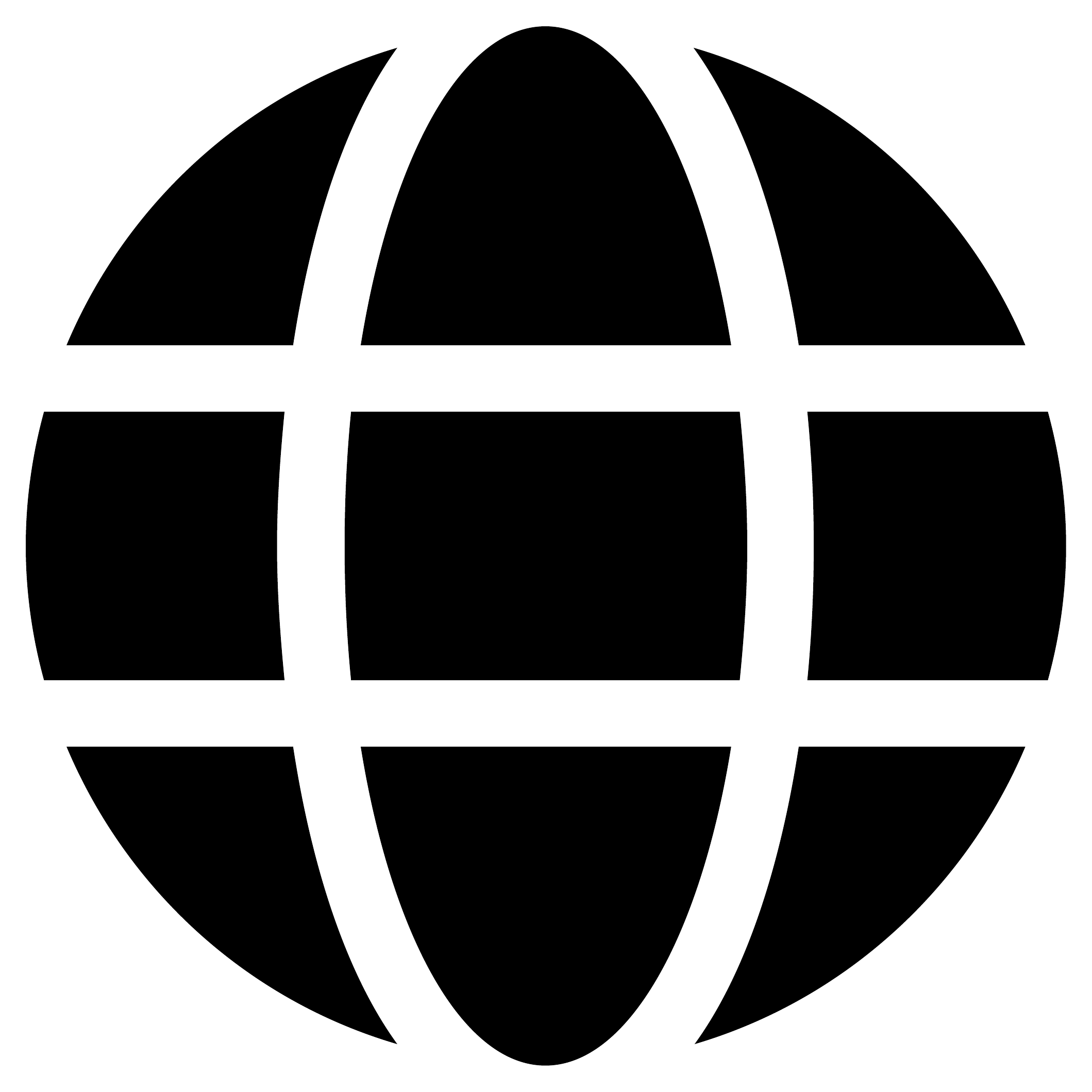}}\hspace{0.4em}\textbf{Webpage}}\hspace{0.9em}{\color{black!25}|}\hspace{0.9em}
\href{https://github.com/CCS-CoTFE/CCS-CoTFE}{\color{black}\raisebox{\dimexpr .5\fontcharht\font`H-.5\height\relax}{\includegraphics[width=1em,height=1em]{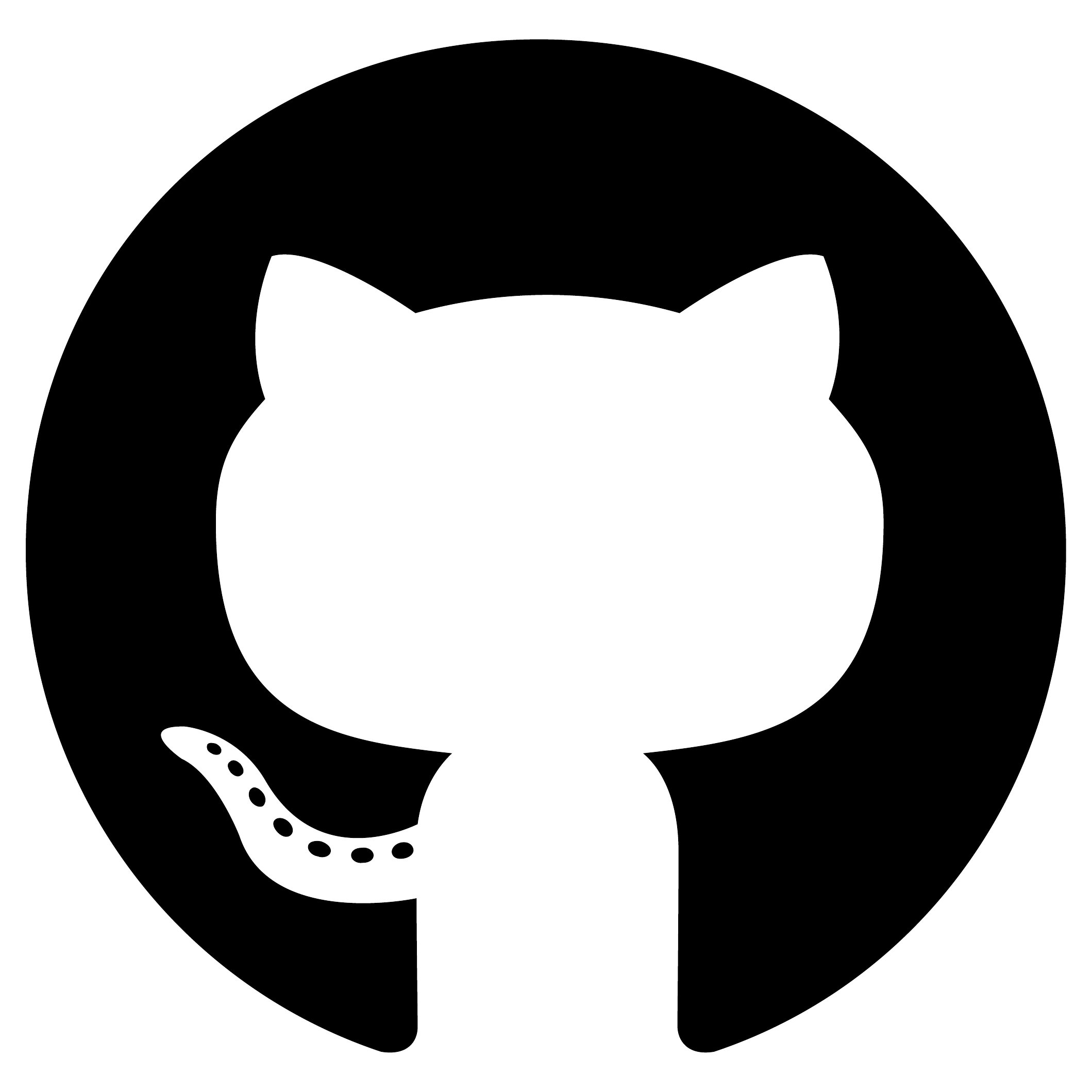}}\hspace{0.4em}\textbf{GitHub}}}}
\date{}

\makeatletter
\renewcommand{\maketitle}{%
  \begin{center}
    {\fontsize{17.28}{20}\selectfont\bfseries\@title\par}
    \vskip 1.2em
    {\fontsize{10}{12}\selectfont\@author\par}
  \end{center}
  \vskip 0.4em
}
\newenvironment{overviewfigure}{%
  \par\medskip\noindent\begin{minipage}{\textwidth}%
  \def\@captype{figure}\centering
}{\end{minipage}\par\medskip}
\makeatother

\begin{document}
\thispagestyle{opening}

% Keep the title, overview and single-paragraph abstract compact on page one.
\begingroup
\setlength{\parskip}{0pt}
\vspace*{-30pt}
\maketitle

% BEGIN sections/overview
% Conceptual illustration above independently plotted empirical results.
\begin{overviewfigure}
\includegraphics[width=\textwidth]{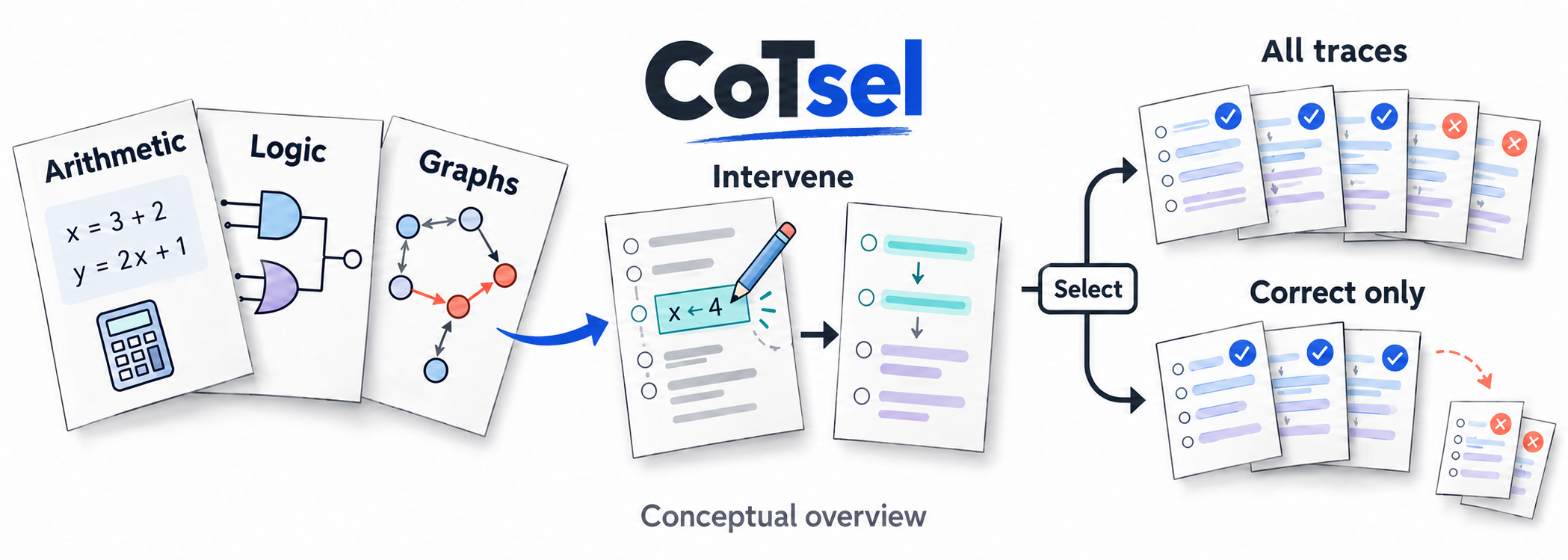}\\[4pt]
\includegraphics[width=\textwidth]{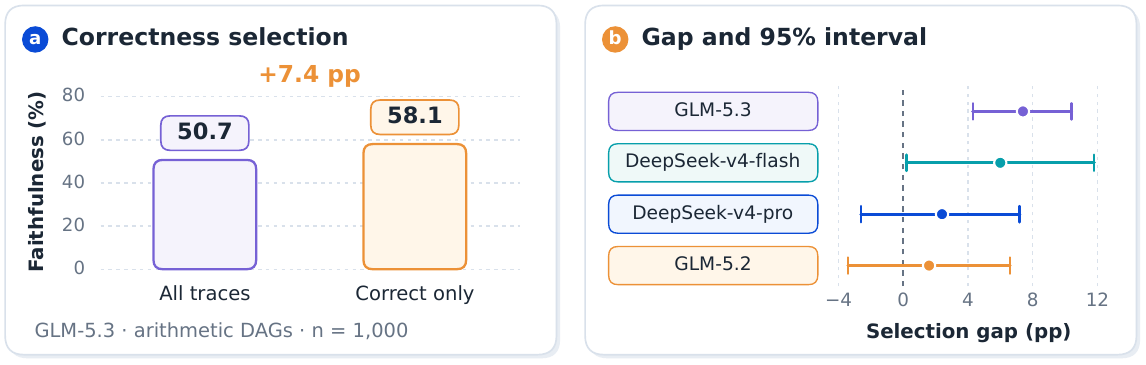}
\caption{\textbf{CoTsel overview.} Conceptual task intervention and correctness selection (top), with measured faithfulness and selection gaps on arithmetic DAGs (bottom).}
\label{fig:overview}
\end{overviewfigure}
% END sections/overview
% BEGIN sections/abstract
\begin{abstract}
Restricting chain-of-thought (CoT) faithfulness evaluation to correctly answered instances changes the population being measured. We study this selection step in a seven-variable structural causal model and derive the exact observed gap $\Delta = \mathrm{Cov}(A, C)/P(C{=}1)$. The model distinguishes difficulty confounding from shortcut-induced collider bias and states the conditional-independence assumption needed to express the gap through latent response propensities. A protocol review of 18 papers distinguishes answer-correctness restrictions from agreement-based selection, benchmark construction, and training-data filtering. We evaluate four sub-ceiling deployments on arithmetic tasks and observe positive selection gaps in all four. The primary GLM-5.3 run yields a $+7.4$ percentage-point gap ($n = 1000$, $p < 0.0001$), and pooled Cochran--Mantel--Haenszel analysis gives OR $= 1.70$ ($p < 10^{-6}$). An assumption-free covariance decomposition assigns 42.8\% of the GLM-5.3 gap to between-stratum variation in interventionable-operation count, leaving the majority unattributed. When counterfactual compliance is observable for every instance, evaluating all traces directly estimates population faithfulness without correctness filtering.
\end{abstract}
% END sections/abstract
\par
\endgroup
\clearpage
% BEGIN sections/introduction
\section{Introduction}
\label{sec:introduction}

Chain-of-thought (CoT) faithfulness evaluation asks whether a model's reasoning trace reflects the computation producing its final answer~\citep{al2023measuring,al2023language}. One approach intervenes on a reasoning step and measures the response. Restricting the comparison to correctly answered instances changes the population being evaluated and, with it, the quantity being estimated. Correct-response analyses occur in published protocols, including Putnam shortcut detection~\citep{arcuschin2025wild} and MMLU step-importance analysis~\citep{bogdan2025anchors}. These protocols measure different properties; our question is how correctness selection changes an intervention-based compliance estimate.

Figure~\ref{fig:overview} summarizes the evaluation protocol and the choice of reporting population. Checks and crosses indicate factual answer correctness; counterfactual compliance is measured separately. The task cards illustrate the protocol families, and the plots report measured results on arithmetic DAGs.

We formalize correctness conditioning in a seven-variable structural causal model (SCM, Figure~\ref{fig:scm}). The identity $\Delta = \Cov(A,C)/P(C{=}1)$ gives the observed selection gap. Under an explicit conditional-independence premise, it factors as $\Delta = \Cov(\eta, \pi) / P(C{=}1)$, relating the gap to correctness and compliance propensities. The framework links protocol and estimand, guiding our audit, deployment tests, and adjustment for an observed difficulty proxy.

The SCM separates two pathways. Difficulty affects both correctness and compliance ($D \to C$, $D \to A$), while conditioning on correctness opens the collider $R \to C \leftarrow S$ when latent shortcut competence is present. Our main experiments measure the association with operation count; the neutral-framing detector records no shortcut use. We derive a ranking-reversal condition and, under the perfect-faithful-reasoning special case, an accuracy-axis peak at $P^{*} = \sqrt{p_s}$. Separate exploratory probes elicit the required shortcut reliance (Appendix~\ref{sec:shortcut_elicitation}), but do not test the peak or reversal thresholds.

Our main contributions are threefold:
\begin{itemize}[leftmargin=1.5em,labelsep=0.6em,itemsep=7pt,topsep=5pt,parsep=0pt]
\item We develop a \textbf{causal framework for correctness conditioning} in CoT faithfulness evaluation. We derive an exact expression for the observed selection gap and specify the assumptions required to analyse difficulty confounding and shortcut-induced collider bias.

\item We compare \textbf{reporting populations in 18 papers}, separating answer-correctness restrictions from agreement-based selection, benchmark construction, and training-data filtering. The review identifies which protocol choices correspond to the selection operator in our analysis.

\item We evaluate \textbf{four model deployments on arithmetic tasks} and measure a \textbf{7.4 percentage-point selection gap} on the primary GLM-5.3 run. We decompose the gap by operation count and test its sensitivity to reweighting and the handling of unparseable responses.
\end{itemize}

We analyse the explicit-reliance probes independently of the main experiments; they use separate endpoints (Appendix~\ref{sec:shortcut_elicitation}). For evaluation, the framework gives a direct rule: use all traces whenever counterfactual compliance is observable for every instance (Section~\ref{sec:theory}).
% END sections/introduction
% BEGIN sections/related_work
\section{Related Work}
\label{sec:related}

\paragraph{Intervention-based CoT faithfulness.}
Chain-of-thought prompting~\citep{wei2022chain} provides an explicit reasoning trace for large language models. Trace evaluations use truncation, paraphrase, and mistake insertion~\citep{al2023measuring}. Injected cues can bias CoT while the reasoning remains plausible~\citep{al2023language}. An analysis of metric degeneracy examines how answer accuracy can confound unfaithfulness measures in one setting~\citep{al2024chain}. More recent work studies faithfulness through information flow~\citep{al2026faithfulness}, counterfactual simulation~\citep{al2026counterfactual}, step-level importance~\citep{bogdan2025anchors}, unlearning~\citep{tutek2025unlearning}, and benchmarking~\citep{al2025faithcot}.

Two 2026 studies identify additional sources of measurement error. One measures alignment between latent answer commitment and the explicit trace (61.9\% average alignment)~\citep{li2026performative}. This analysis concerns trace timing, whereas our estimand depends on which traces enter the aggregate. The other shows that classifier choice alone reverses rankings (regex 74.4\% vs.\ LLM judge 69.7\% on identical data)~\citep{young2026measuring}. This changes measurement on a fixed sample; correctness filtering changes the sample itself. Accurate step-level assessment and a fixed classifier therefore do not resolve the selection problem. Conversely, including every trace does not resolve classifier sensitivity. We derive the bias introduced by correctness conditioning in closed form.

Causal approaches also examine the reasoning process and its inputs. Controlled perturbations serve as instrumental signals for assessing step-level faithfulness within a single chain (intra-chain dependence)~\citep{sun2026facte}, while causal models of CoT steps address correctness and interpretability~\citep{fu2025unveiling}. Studies of vision-language models examine how text and image biases affect faithfulness articulation~\citep{balasubramanian2025closer}. Confidence-based trace filtering has been proposed as an alternative evaluation strategy~\citep{pathak2026filtered}. Our analysis addresses the evaluation \emph{sample}: conditioning it on answer correctness changes the aggregate faithfulness measure.

\paragraph{Selection bias in evaluation.}
Studies of LLM evaluation address position effects~\citep{shi2025judging} and metric disagreement. Collider bias has been identified in LLM training data curation~\citep{mcmilin2022collider}; correctness filtering introduces a related selection step at evaluation time. Other studies assess the causal validity of faithfulness metrics~\citep{zaman2025causal} and their accuracy against ground truth~\citep{bonafide2026}. We complement these analyses by characterizing selection in the sample on which the metric is computed.

Our analysis connects to sample-selection models~\citep{heckman1979} and, when the shortcut pathway is active, collider bias~\citep{elwert2014endogenous}. Correctness filtering also raises a recoverability question when compliance on excluded instances is unavailable~\citep{rubin1976inference,mohan2021graphical}. Related settings include selective classification with abstention~\citep{elyaniv2010foundations,geifman2017selective}, the selective-labels problem~\citep{lakkaraju2017selective}, and graphical recoverability criteria~\citep{bareinboim2012controlling,bareinboim2014transportability}. These methods model selection explicitly. We use the same perspective to diagnose selection in a correctness-conditioned CoT evaluation. The distinction between faithfulness and self-consistency~\citep{parcalabescu2024measuring} and tests of explanation faithfulness~\citep{atanasova2023faithfulness} clarify metric semantics, complementing our sample-selection analysis.
\paragraph{Causal inference and shortcut learning.}
Our SCM follows Pearl's structural approach~\citep{pearl2009causality}, and the intervention analysis builds on causal-mediation methods for neural models~\citep{vig2020causal}. The collider $R \to C \leftarrow S$ is an instance of Berkson's paradox~\citep{berkson1946limitations}. Conditioning on post-treatment variables can bias observational studies~\citep{rosenbaum2002observational}, and our variable $S$ maps to the taxonomy of shortcut learning~\citep{geirhos2020shortcut}. Evidence that LLMs use hidden cues~\citep{chen2025reasoning} and that training against monitored reasoning can promote obfuscation~\citep{baker2025monitorability} motivates including shortcut competence in the SCM.
% END sections/related_work
% BEGIN sections/theory
\section{A Diagnostic Framework for Selection Bias}
\label{sec:theory}

We model correctness conditioning as selection in a structural causal model (SCM). We first define the evaluation variables and derive the selection gap, then analyse difficulty confounding and shortcut-induced collider bias under their respective assumptions.

Figure~\ref{fig:response_rules} illustrates why factual correctness alone does not establish counterfactual compliance. Both rules answer the original task correctly. After the edit, recomputing the dependency returns the new oracle answer, while copying the unchanged cue fails. The rules illustrate the evaluation protocol; they are not recorded model responses or latent types identified in the deployment data.

\begin{figure}[!htbp]
\centering
\includegraphics[width=\textwidth]{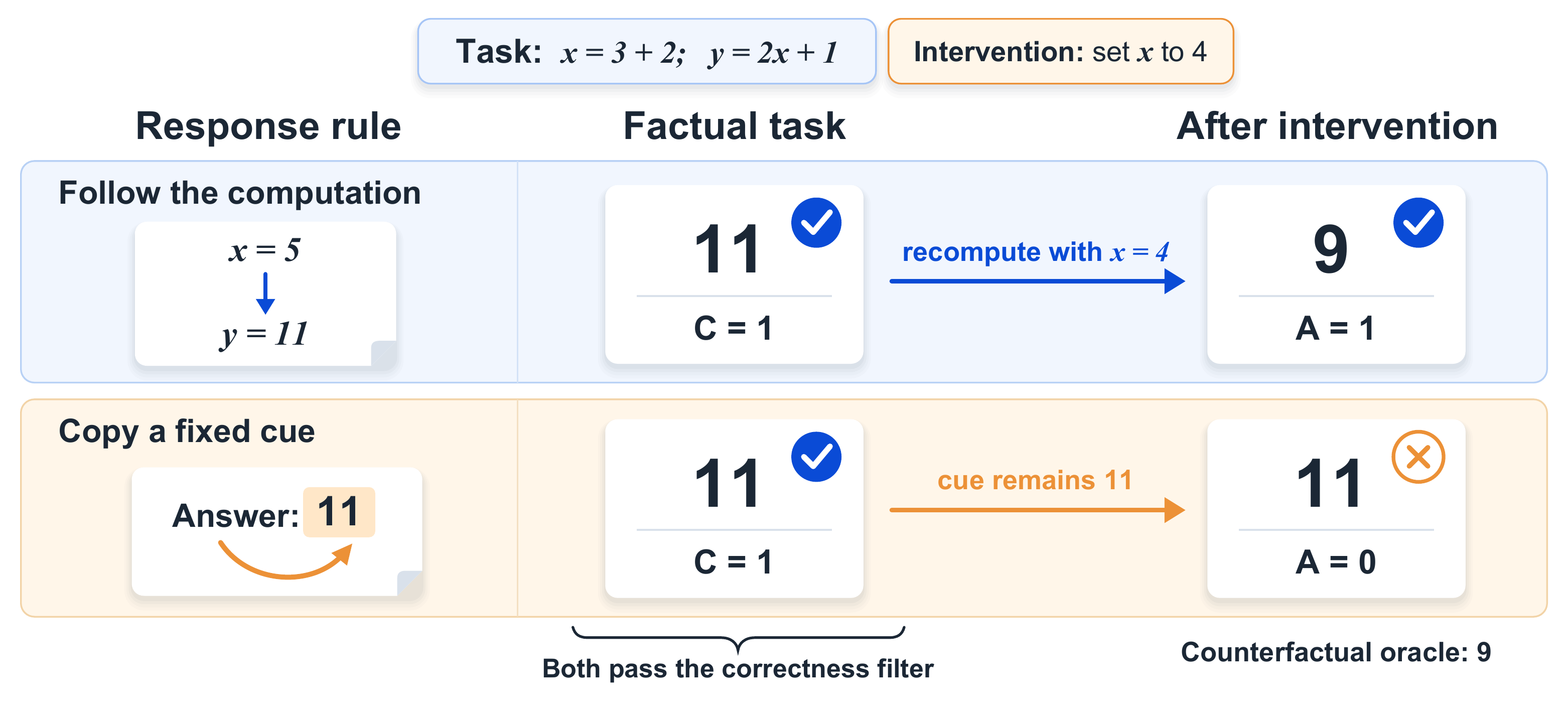}
\caption{Correct factual answers can conceal different intervention responses. The two rules agree before the edit and diverge afterward.}
\label{fig:response_rules}
\end{figure}

\subsection{Structural Causal Model}
\label{sec:scm}

\begin{definition}[Faithfulness Evaluation SCM]
\label{def:scm}
Let $(\mathbf{U}, \mathbf{V}, \mathcal{F})$ be a stochastic SCM with endogenous variables $\mathbf{V} = \{D, S, R, T, C, I, A\}$ and mutually independent exogenous noise $\mathbf{U} = \{U_D, U_S, U_R, U_T, U_C, U_I, U_A\}$. The Bernoulli kernels below are shorthand for structural functions of the corresponding independent noise terms:
\begin{align}
D &= f_D(U_D), \label{eq:D} \\
S &\sim \mathrm{Bernoulli}(p_S(D)), \label{eq:S} \\
R &\sim \mathrm{Bernoulli}(p_R(D)), \label{eq:R} \\
T &\sim \mathrm{Bernoulli}(p_T(R, D)), \label{eq:T} \\
C &\sim \mathrm{Bernoulli}(\pi(R, S, T, D)), \label{eq:C} \\
I &\sim \mathrm{Bernoulli}(\rho), \label{eq:I} \\
A &\sim \mathrm{Bernoulli}(\eta(R, S, T, D, I)), \label{eq:A}
\end{align}
where $D$ is task difficulty, $S$ is shortcut competence (whether the model has a non-chain-of-thought route to the answer), $R$ is chain-of-thought (CoT) reliance (whether the model uses the visible reasoning trace), $T$ is trace correctness, $C$ is answer correctness (the collider), $I$ is randomized intervention assignment, and $A$ is appropriate post-intervention answer change (the compliance outcome).
\end{definition}

The Bernoulli kernels represent binary outcomes. Filtering on $C=1$ retains traces from correctly answered tasks; $A$ records whether the answer changes appropriately after an intervention on the reasoning chain. Two structures can therefore produce a gap between the selected-trace and population estimates:

\begin{equation}
\underbrace{R \to C \leftarrow S}_{\text{collider (needs } S>0\text{)}} \qquad \text{and} \qquad \underbrace{D \to C,\; D \to A}_{\text{difficulty confounding}}.
\label{eq:both_paths}
\end{equation}

Section~\ref{sec:difficulty} relates the confounding structure to the measured covariate. Section~\ref{sec:collider_case} derives the collider pathway under an active shortcut branch; the neutral-framing runs provide no detected shortcut use with which to test it (note $T \perp S \mid R, D$).

In Figure~\ref{fig:scm}, orange edges identify the collider $R\to C\leftarrow S$, teal edges identify the difficulty paths $D\to C$ and $D\to A$, and gray edges show the remaining structural dependencies. The double border marks conditioning on $C=1$; the dashed red line denotes an association that conditioning can open between reliance and shortcut competence within a fixed difficulty stratum. It is not a causal edge.

\begin{figure}[t]
\centering
\includegraphics[width=\textwidth]{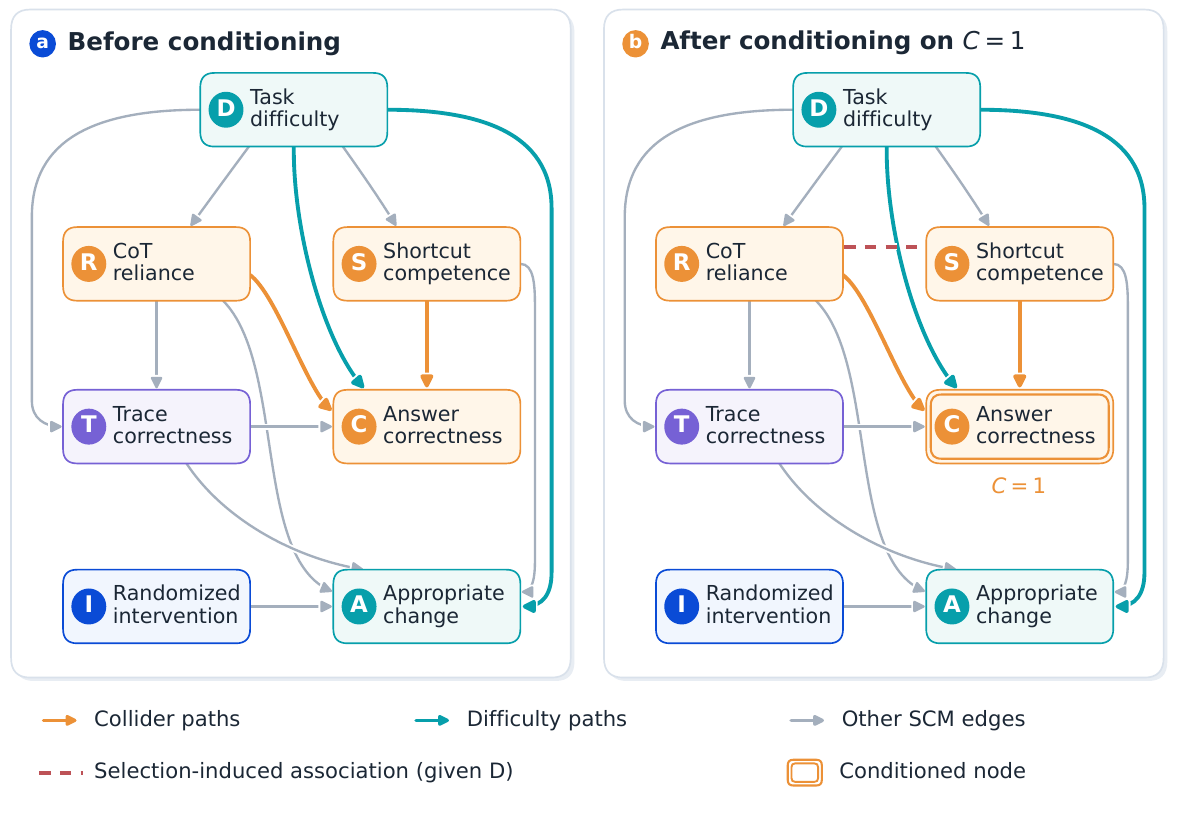}
\caption{Structural causal model before and after correctness conditioning. Both panels show all dependencies in Definition~\ref{def:scm}; the dashed line marks a selection-induced association.}
\label{fig:scm}
\end{figure}

\subsection{The General Bias Formula}
\label{sec:bias}

Within a fixed difficulty stratum $D = d$, define the $T$-marginalized cell probabilities:
\begin{align}
\pi_{rs} &\coloneqq P(C = 1 \mid R = r, S = s, D = d), \label{eq:pi_rs} \\
\eta_{rs} &\coloneqq P(A = 1 \mid R = r, S = s, I = 1, D = d). \label{eq:eta_rs}
\end{align}
Assume $R \perp S \mid D = d$ with $P(R = 1 \mid D = d) = p_r$ and $P(S = 1 \mid D = d) = p_s$. In our SCM, $R = f_R(D, U_R)$ and $S = f_S(D, U_S)$ where $U_R$ and $U_S$ are independent exogenous variables. The Markov property gives $R \perp S \mid D$: reasoning ability and shortcut competence share only task difficulty as a common cause.

\begin{definition}[Population and Selected-Trace Faithfulness]
\label{def:faithfulness}
For causal interventions ($I = 1$), within a fixed difficulty stratum $D = d$:
\begin{align}
\Fpop^{(d)} &\coloneqq P(A = 1 \mid I = 1, D = d), \label{eq:fpop_d} \\
\Fsel^{(d)} &\coloneqq P(A = 1 \mid I = 1, C = 1, D = d), \label{eq:fsel_d} \\
\Delta^{(d)} &\coloneqq \Fsel^{(d)} - \Fpop^{(d)}. \label{eq:delta_d}
\end{align}
The marginal analogues reported in Table~\ref{tab:empirical} drop the stratification, with $\Fpop$, $\Fsel$, and $\Delta \coloneqq \Fsel - \Fpop$ defined by the same expressions without conditioning on $D$.
\end{definition}

\begin{proposition}[General Bias Formula]
\label{prop:collider_bias}
Assume $A \perp C \mid \{R, S, I, D\}$ (compliance is conditionally independent of correctness given the strategy type and intervention). Then the selection bias equals the covariance of intervention-response propensity and correctness propensity, normalized by the selection probability:
\begin{equation}
\Delta = \frac{\Cov(\eta_{RS}, \pi_{RS} \mid D = d)}{P(C = 1 \mid D = d)},
\label{eq:delta_cov}
\end{equation}
where $\eta_{RS}$ and $\pi_{RS}$ are random cell-level functions of $(R, S)$ under the population distribution.
\end{proposition}

\begin{proof}
Expanding $\Fsel$ using Bayes' rule:
\begin{equation}
\Fsel = \sum_{r,s} \eta_{rs} \cdot \frac{\pi_{rs} \, P(R{=}r) \, P(S{=}s)}{P(C{=}1)} = \frac{E[\eta_{RS} \, \pi_{RS}]}{P(C{=}1)}.
\end{equation}
Since $\Fpop = E[\eta_{RS}]$ and $P(C{=}1) = E[\pi_{RS}]$:
\begin{equation}
\Delta = \frac{E[\eta_{RS} \, \pi_{RS}]}{E[\pi_{RS}]} - E[\eta_{RS}]
= \frac{E[\eta_{RS} \, \pi_{RS}] - E[\eta_{RS}] \, E[\pi_{RS}]}{E[\pi_{RS}]}
= \frac{\Cov(\eta_{RS}, \pi_{RS})}{P(C{=}1)}.
\end{equation}
\end{proof}

The observed selection gap has an assumption-free expression. With $\Fsel = E[AC]/E[C]$ and $\Fpop = E[A]$,
\begin{equation}
\Delta = \frac{E[AC]}{E[C]} - E[A] = \frac{E[AC] - E[A]\,E[C]}{E[C]} = \frac{\Cov(A, C)}{P(C{=}1)},
\label{eq:bridge}
\end{equation}
using $P(C{=}1) = E[C]$. Equation~\ref{eq:bridge} is an algebraic identity on the observed pair $(A, C)$, with no latent-variable or independence assumption. This is the gap estimated in our experiments. The factorization in \eqref{eq:delta_cov} requires the additional premise $A \perp C \mid \{R, S, I, D\}$: trace correctness $T$ enters both \eqref{eq:C} and \eqref{eq:A}, leaving $C \leftarrow T \rightarrow A$ open after conditioning on $\{R, S, I, D\}$. The premise therefore constrains the SCM rather than following from it. With this constraint, $\Cov(\eta_{RS}, \pi_{RS}) = \Cov(A, C)$; without it, the two covariances can differ.

Stratifying on an observed variable gives a separate, assumption-free decomposition. With $D$ the measured difficulty proxy, the law of total covariance gives
\begin{equation}
\Cov(A, C) = \underbrace{\Cov\!\left(E[A \mid D],\, E[C \mid D]\right)}_{\text{between-stratum}} + \underbrace{E\!\left[\Cov(A, C \mid D)\right]}_{\text{within-stratum}},
\label{eq:total_cov}
\end{equation}
which holds for any joint distribution of $(A, C, D)$. We use \eqref{eq:total_cov} in Section~\ref{sec:empirical} to measure variation between strata of the observed proxy. The within-stratum term collects residual covariance without a mechanistic attribution.

Correctness filtering is unbiased when $\Cov(A,C) = 0$. This includes the boundary $P(C{=}1) = 1$ and, under the conditional-independence premise, the case in which $\eta_{RS}$ is constant across cells. Equation~\eqref{eq:delta_cov} expresses the gap through cell-level propensities under that premise, while \eqref{eq:bridge} describes the observed gap directly. The theory identifies the stratum-specific $\Delta^{(d)}$; the marginal $\Delta$ of Table~\ref{tab:empirical} additionally includes the between-stratum composition term (Appendix~\ref{sec:estimand}).

\subsection{Difficulty Confounding}
\label{sec:difficulty}

When task difficulty $D$ reduces both correctness ($D \to C$) and compliance ($D \to A$), conditioning on $C = 1$ selects easier tasks with higher compliance. We measure the associated stratified variation using operation count as a difficulty proxy (Section~\ref{sec:empirical}). Between-stratum variation accounts for 21.9--42.8\% of the observed gap. This is a descriptive share; identifying the causal contribution of latent difficulty requires varying difficulty while holding other causes of correctness fixed.

\begin{proposition}[Difficulty-Confounding Bias]
\label{prop:difficulty}
Let $\pi(d) = P(C = 1 \mid D = d)$ and $\eta(d) = P(A = 1 \mid I = 1, D = d)$. Assume $A \perp C \mid \{D, I\}$. When both $\pi$ and $\eta$ decrease with difficulty (harder tasks have lower accuracy and lower compliance), conditioning on $C = 1$ selects easier instances and inflates the faithfulness estimate:
\begin{equation}
\Delta = \frac{\Cov(\eta(D), \pi(D))}{P(C = 1)},
\label{eq:delta_difficulty}
\end{equation}
the same covariance form as Proposition~\ref{prop:collider_bias} with $(R,S)$ replaced by the observed variable $D$. Whenever harder tasks produce both lower accuracy and lower compliance ($\mathrm{Cov}(\eta(D), \pi(D)) > 0$), the bias is strictly positive.
\end{proposition}

Proposition~\ref{prop:difficulty} holds under the additional assumption $A \perp C \mid \{D, I\}$, satisfied when compliance depends only on task difficulty and intervention type. This premise excludes residual dependence beyond $D$ and must be imposed in addition to \eqref{eq:C}--\eqref{eq:A}. With a measured difficulty proxy, we estimate within-stratum correctness and compliance for the adjustment in Section~\ref{sec:empirical}. All-trace evaluation removes correctness selection for either pathway.

\subsection{Shortcut Competence and Collider Bias}
\label{sec:collider_case}

With latent shortcut competence $S > 0$, conditioning on correctness opens the collider $R \to C \leftarrow S$. Under the perfect-faithful-reasoning special case (Corollary~\ref{cor:simplified}), the gap decreases monotonically in shortcut prevalence (Corollary~\ref{cor:ps_monotone}) and peaks at $p_r^* = \sqrt{p_s}/(1+\sqrt{p_s})$ in CoT reliance (Corollary~\ref{cor:hump}). Proposition~\ref{prop:reversal} gives the conditions for reversing two models' faithfulness rankings. These predictions concern an active shortcut branch. The neutral-framing detector records $0\%$ use (\S\ref{sec:empirical}); separate explicit-reliance probes elicit shortcut use on several deployments (Appendix~\ref{sec:shortcut_elicitation}), but do not test the peak or reversal predictions. Appendix~\ref{sec:collider_full} gives the derivations and proofs.

\subsection{Implications for Evaluation Practice}
\label{sec:implications}

Evaluators should report the population on which a faithfulness score is computed. Correctness selection leaves the score unchanged when compliance and correctness are uncorrelated, including perfect accuracy or, under the stated premise, constant compliance. Under the special case in Corollary~\ref{cor:simplified}, the gap peaks at intermediate CoT reliance (Corollary~\ref{cor:hump}).

When compliance can be measured for every instance, including incorrect answers, evaluating all traces computes $\Fpop$ directly. Most natural-language settings lack a counterfactual oracle. Evaluators should then state the correctness condition and measurement limits, reporting $\Fsel$ and $\Fpop$ when both are estimable.
% END sections/theory
% BEGIN sections/audit_body
\section{Reporting Populations in Prior Evaluations}
\label{sec:audit}

We inspect 18 papers on CoT faithfulness and related evaluation protocols. The set spans trace interventions, input perturbations, benchmark construction, and training procedures. We record the population used by each inspected analysis rather than assign a single filtering label to an entire paper. Table~\ref{tab:audit_full} gives version-specific source locations; Figure~\ref{fig:audit_landscape} summarizes the distinctions.

\paragraph{Review procedure.} We inspect metric definitions, evaluation sections, and relevant appendices for the 18 papers in Table~\ref{tab:audit_full}. For each protocol, we distinguish selection of training examples, construction of a benchmark, and restriction of the final evaluation denominator. This targeted review compares concrete protocol choices; a prevalence estimate would require a systematic sampling frame and a common unit of analysis.

\paragraph{Answer correctness.} Arcuschin et al. restrict their Putnam shortcut analysis to correct responses, while their separate comparison-pair analysis uses all retained input pairs~\citep{arcuschin2025wild}. Bogdan et al. use correct CoTs for MMLU step-importance analysis but outcome-balanced traces for their MATH analysis~\citep{bogdan2025anchors}. Each analysis has its own reporting population. A correct-only score estimates a conditional quantity; our analysis quantifies its difference from the population score.

\paragraph{Other selection rules.} Tutek et al. restrict their faithfulness comparison to instances on which CoT and no-CoT predictions agree~\citep{tutek2025unlearning}. Agreement does not require either answer to be correct. Chen et al. select pairs where adding a hint changes the answer to that hint~\citep{chen2025reasoning}; correctness then depends on the hint's validity. Bentham et al. average with/without-CoT agreement over the dataset~\citep{al2024chain}. Their accuracy-confound analysis does not establish a correctness-filtered denominator.

\begin{figure}[!htbp]
\centering
\includegraphics[width=\textwidth]{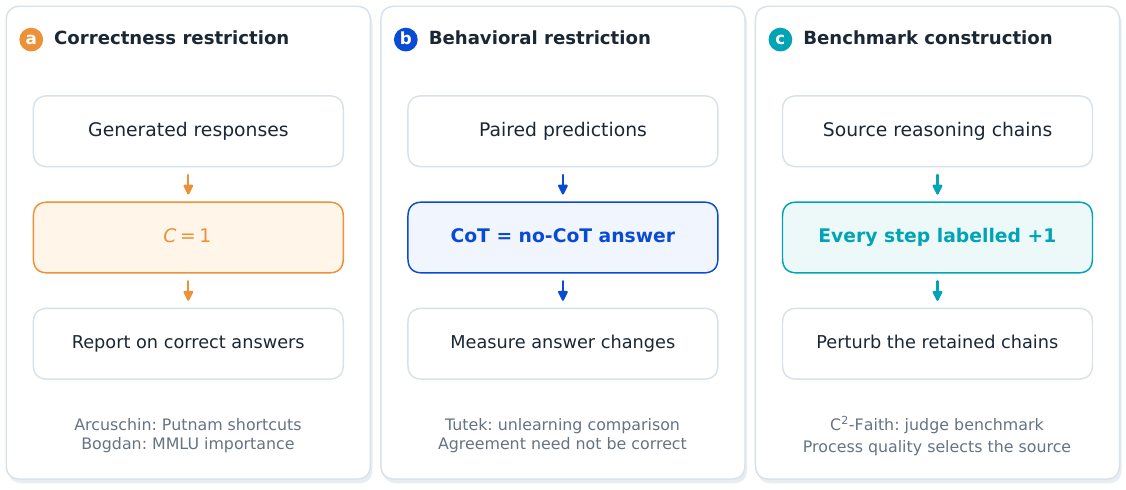}
\caption{\textbf{Selection rules in evaluation protocols.} Answer correctness, behavioral agreement, and data construction define different populations.}
\label{fig:audit_landscape}
\end{figure}

\paragraph{Construction and training.} C$^2$-Faith perturbs chains whose steps all have positive human labels~\citep{c2faith2026}; this process-quality criterion is distinct from filtering an evaluated model's final answers. Bona Fide constructs ground truth from correct outright answers or responses that follow a wrong hint~\citep{bonafide2026}. Correct-answer filtering also appears in the training pool of Knowledge-Driven CoT~\citep{al2023knowledge}. Comparing these protocols requires specifying when selection occurs and which population the score describes. Appendix~\ref{sec:audit_rule} details these distinctions.
% END sections/audit_body
% BEGIN sections/methodology
\section{Experimental Design}
\label{sec:method}

We compare the correctness-filtered estimator $\Fsel$ with the all-trace estimator $\Fpop$ on procedural tasks with known causal structure. Deterministic counterfactual answers make compliance observable for every task, and configurable cues allow separate tests of shortcut reliance.

\subsection{Procedural Task Design}
\label{sec:tasks}

Intervention-based evaluation requires a reference answer for the edited task. We construct procedural tasks in which a symbolic executor replays the dependency graph and computes this counterfactual answer exactly~\citep{al2023faithful,al2026counterfactual}. The resulting oracle supports the same compliance measurement on correctly and incorrectly answered instances. Natural-language tasks generally lack a unique, verifiable answer after an edit; extending the results to that setting requires a separate evaluation.

We implement three task families with deterministic oracles and semantics-preserving placebo edits: \textbf{arithmetic DAGs} of integer operations (depth 2--4), \textbf{Boolean circuits} of logical gates (depth 2--8), and \textbf{graph reachability} on 5--12 nodes. Each family exposes an editable dependency structure. The pilot generators produce 100 candidate instances per family from deterministic seeds. We use smaller batches for the implementation check (Table~\ref{tab:simulation}) and separate run sizes for the deployment study (Table~\ref{tab:empirical}). A faithful baseline ($n = 50$) validates the evaluation procedure, giving $\Fpop = \Fsel = 1.000$ (Appendix~\ref{sec:measurement}).

Figure~\ref{fig:task_atlas} gives a worked example from each family. Arithmetic and Boolean edits override an intermediate value before replaying downstream operations. Graph edits change an input edge and recompute reachability. These examples illustrate the generators; the headline deployment results use arithmetic DAGs.

\begin{figure}[!htbp]
\centering
\includegraphics[width=\textwidth,trim=0 74bp 0 0,clip]{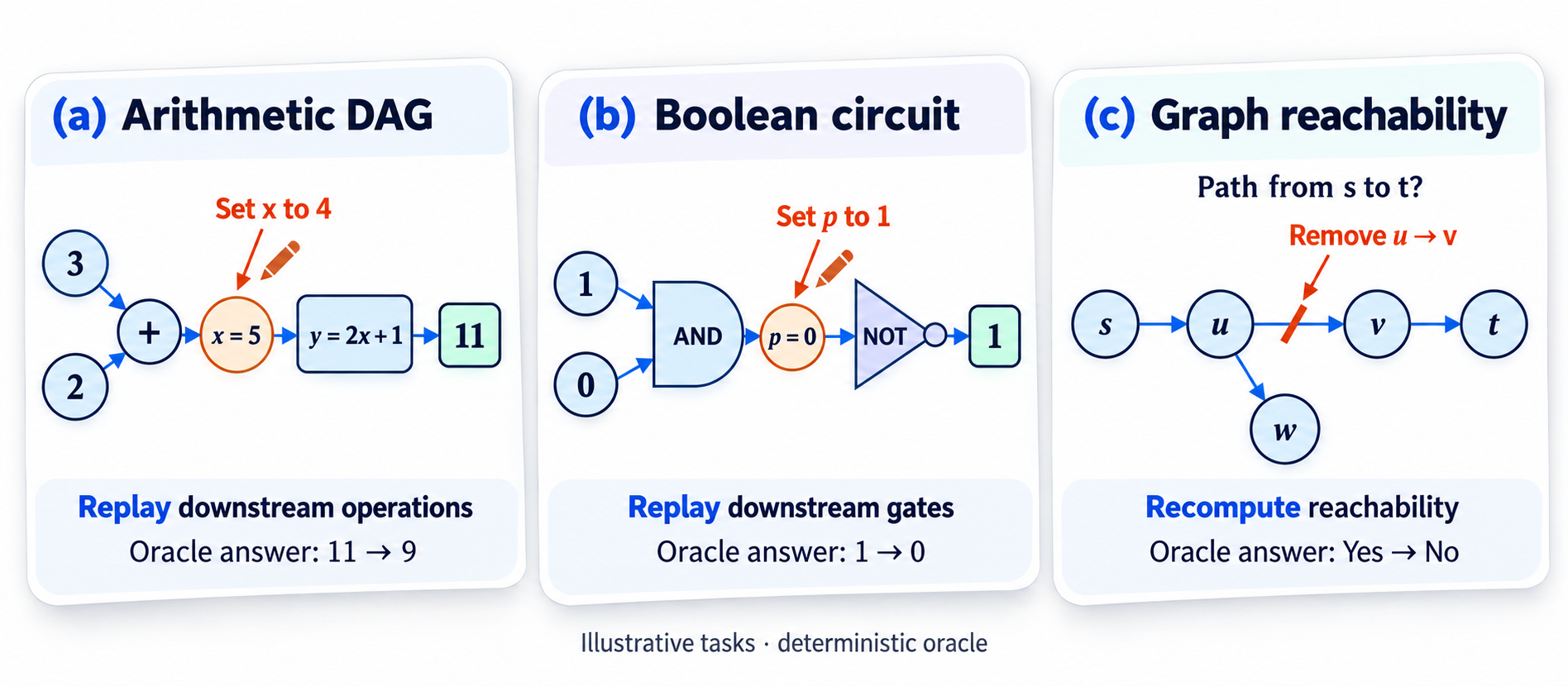}
\caption{\textbf{Procedural task families.} Each example pairs the original task with an intervention and a deterministic counterfactual answer.}
\label{fig:task_atlas}
\end{figure}

\paragraph{Planted shortcut channels.}
\label{par:shortcut_design}
In the procedural tasks, each computed step depends on earlier values. Without a planted cue, the design targets $S \approx 0$, suppressing the shortcut pathway; $\Delta = 0$ additionally requires the absence of other correctness--compliance dependence. To study shortcut use, we augment instances with a cue whose validity is controlled by $q$, the probability that it agrees with the oracle answer. The cue is not updated under intervention, so a cue-reliant model can pass factual correctness while failing counterfactual compliance. The generator control $q$ is distinct from the theoretical $p_s = P(S{=}1\mid D)$, which denotes the prevalence of shortcut-capable behavior in the SCM and is set directly in the implementation check.

\subsection{All-Trace Intervention Protocol}
\label{sec:protocol}

For each instance $i$, the language model generates a trace and answer, and the symbolic executor supplies the reference answer. We present three conditions independently. The \textbf{factual condition} uses the original task and records correctness $C_i = \mathbb{1}[\hat{y}_i^{\mathrm{fact}} = y_i^*]$. The \textbf{causal intervention condition} changes one intermediate graph value; the oracle replays the downstream computation to obtain $y_i^{\mathrm{cf}*}$, and we record compliance $A_i = \mathbb{1}[\hat{y}_i^{\mathrm{cf}} = y_i^{\mathrm{cf}*}]$. The \textbf{placebo condition} applies a semantics-preserving edit and records $A_i^{\mathrm{plac}} = \mathbb{1}[\hat{y}_i^{\mathrm{plac}} = y_i^*]$.

Each condition is a standalone task. In the headline analysis, we score non-interventionable traces with missing or unparseable answers as non-compliant ($A_i = 0$). We also repeat the analysis with these responses treated as missing (Appendix~\ref{sec:compliance_discussion}). The headline convention gives a larger estimated gap.

Figure~\ref{fig:evaluation_protocol} illustrates the three conditions and their oracle-based scores. The example is schematic; the placebo condition validates the instance set and is not a per-model result.

\begin{figure}[!htbp]
\centering
\includegraphics[width=\textwidth]{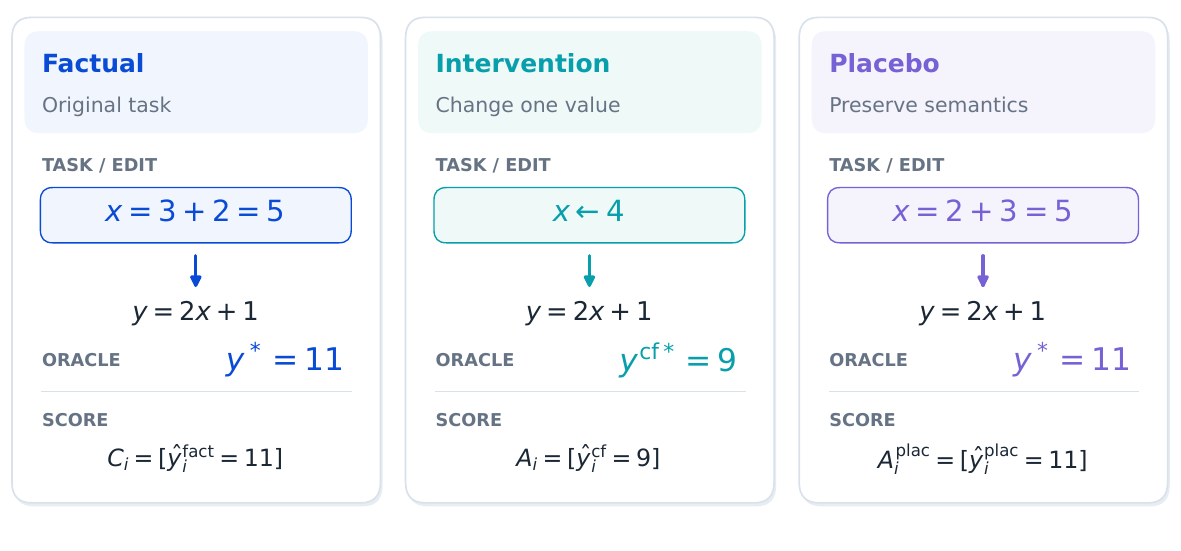}
\caption{\textbf{All-trace intervention protocol.} On this illustrative arithmetic task, a value intervention changes the oracle answer and a semantics-preserving edit leaves it unchanged.}
\label{fig:evaluation_protocol}
\end{figure}

\subsection{Evaluation Metrics}
\label{sec:metrics}

The primary estimands are:
\begin{align}
\Fpop &= \frac{1}{n} \sum_{i=1}^{n} A_i, \qquad
\Fsel = \frac{\sum_{i=1}^{n} A_i \cdot C_i}{\sum_{i=1}^{n} C_i}, \qquad
\Delta = \Fsel - \Fpop. \label{eq:estimands}
\end{align}

We estimate confidence intervals with a nonparametric bootstrap ($B = 10{,}000$ resamples of \emph{instances}, the unit at which the model was queried). The pre-specified primary analysis is a single pooled Cochran--Mantel--Haenszel test. Individual model $p$-values are descriptive; we apply no multiplicity adjustment to the single primary hypothesis. All headline model runs use temperature $= 0$ (greedy decoding).
% END sections/methodology
% BEGIN sections/empirical
\section{Deployment Results}
\label{sec:empirical}

We evaluate four sub-ceiling deployments from ZhipuAI and DeepSeek on arithmetic DAGs, alongside two 100\%-accuracy boundary controls from Alibaba and Google (Table~\ref{tab:empirical}). We first measure the selection gap, then decompose it by operation count and test its sensitivity to adjustment and response handling.

\paragraph{Correctness conditioning raises the point estimates.}
All four sub-ceiling runs have positive selection gaps. Figure~\ref{fig:deployment_gaps} compares the all-trace and selected estimates, with gap intervals at the precision reported in Table~\ref{tab:empirical}. Two intervals cross zero.

\begin{figure}[!htbp]
\centering
\includegraphics[width=\textwidth]{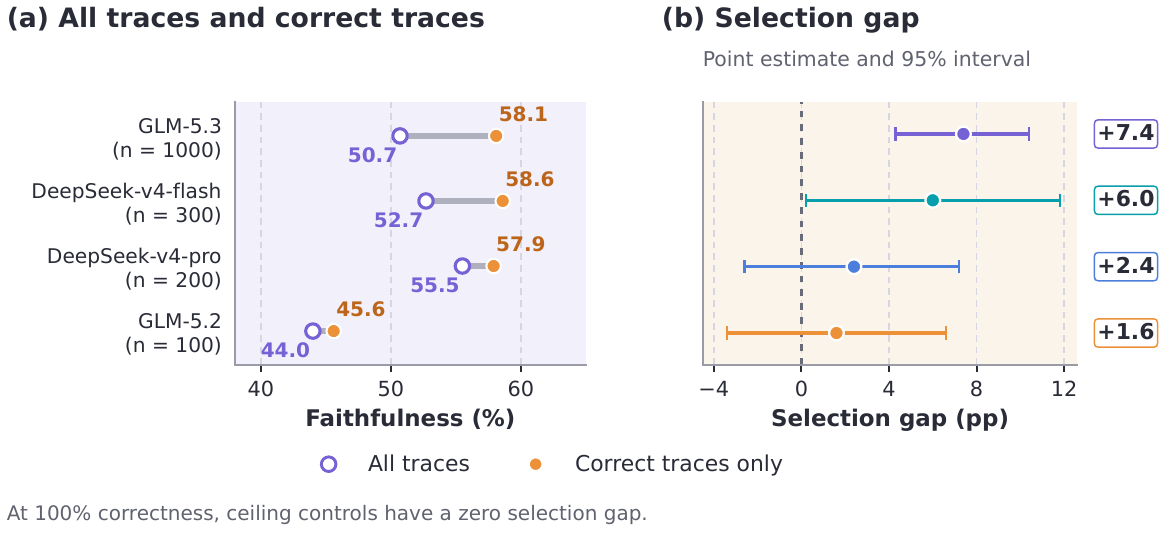}
\caption{\textbf{Faithfulness and selection gaps across deployments.} Open markers denote all traces; filled markers denote selected traces. Intervals show uncertainty in the selection gap.}
\label{fig:deployment_gaps}
\end{figure}

Table~\ref{tab:empirical} reports results for the serving deployments that answered our requests (Appendix~\ref{sec:limitations_full}). We retain missing responses in the run sizes and score them as non-compliant, following Section~\ref{sec:protocol}. Thus $\Fpop$ uses the full run, while $\Fsel$ uses the correctness-selected subset. Table~\ref{tab:twobytwo} gives the joint counts, and asterisks mark two-sided $p<0.05$. The ceiling controls satisfy $P(C=1)=1$ by construction, so their zero gap is an algebraic identity. The Gemini control includes only valid responses (Appendix~\ref{sec:ceiling}).

\begin{table}[htbp]
\caption{\textbf{Faithfulness across serving deployments.} All-trace and correctness-conditioned estimates, with ceiling controls shown separately.}
\label{tab:empirical}
\centering
\fontsize{9}{10.8}\selectfont
\setlength{\tabcolsep}{2.5pt}
\renewcommand{\arraystretch}{1.18}
\arrayrulecolor{black!35}
\setlength{\heavyrulewidth}{0.65pt}
\setlength{\lightrulewidth}{0.35pt}
\begin{tabularx}{\linewidth}{@{}>{\raggedright\arraybackslash}Xccccccc@{}}
\toprule
\rowcolor[HTML]{F1F6FE}
\textbf{\boldmath Model} & \textbf{\boldmath $n$} & \textbf{\boldmath Acc.} & \textbf{\boldmath $\Fpop$} & \textbf{\boldmath $\Fsel$} & \textbf{\boldmath $\Delta$ (pp)} & \textbf{\boldmath 95\% CI (pp)} & \textbf{\boldmath $p$} \\

\midrule
glm-5.2 & 100 & 79.0\% & \cellcolor[HTML]{F5F3FC}0.440 & \cellcolor[HTML]{FFF6E9}0.456 & \cellcolor[HTML]{FFF0DC}\color[HTML]{1B2735}+1.6 & [$-$3.4, +6.6] & 0.540 \\
deepseek-v4-pro & 200 & 66.5\% & \cellcolor[HTML]{F5F3FC}0.555 & \cellcolor[HTML]{FFF6E9}0.579 & \cellcolor[HTML]{FFF0DC}\color[HTML]{1B2735}+2.4 & [$-$2.6, +7.2] & 0.337 \\
\rowcolor[HTML]{EAF1FC}
\textbf{glm-5.3} & \textbf{1000} & \textbf{51.5\%} & \cellcolor[HTML]{F5F3FC}\textbf{0.507} & \cellcolor[HTML]{FFF6E9}\textbf{0.581} & \cellcolor[HTML]{FFF0DC}\color[HTML]{1B2735}\textbf{+7.4} & \textbf{[+4.3, +10.4]} & $\boldsymbol{<}$\textbf{0.0001}$^\ast$ \\
\textbf{deepseek-v4-flash} & \textbf{300} & \textbf{48.3\%} & \cellcolor[HTML]{F5F3FC}\textbf{0.527} & \cellcolor[HTML]{FFF6E9}\textbf{0.586} & \cellcolor[HTML]{FFF0DC}\color[HTML]{1B2735}\textbf{+6.0} & \textbf{[+0.2, +11.8]} & \textbf{0.046}$^\ast$ \\
\midrule
\rowcolor[HTML]{F5F6F7}
\multicolumn{8}{l}{\emph{Boundary controls ($P(C{=}1)=1$, filter vacuous):}} \\
qwen3.8-max & 20 & 100\% & \cellcolor[HTML]{F5F3FC}0.600 & \cellcolor[HTML]{FFF6E9}0.600 & \cellcolor[HTML]{F5F6F7}0.0 & N/A & N/A \\
Gemini 3.5 Flash Lite & 17 & 100\% & \cellcolor[HTML]{F5F3FC}0.529 & \cellcolor[HTML]{FFF6E9}0.529 & \cellcolor[HTML]{F5F6F7}0.0 & N/A & N/A \\
\bottomrule
\end{tabularx}
\end{table}

The primary GLM-5.3 run gives a +7.4pp gap ($p < 0.0001$, $n = 1000$; Figure~\ref{fig:deployment_gaps}), and deepseek-v4-flash gives +6.0pp ($p = 0.046$, $n = 300$).\footnote{Individual model $p$-values are descriptive and two-sided (Pearson $\chi^2$, $df{=}1$); the pre-specified primary analysis is the pooled CMH test.} The remaining two gaps are positive and non-significant. Pooling the four deployments with the Cochran--Mantel--Haenszel test gives OR $= 1.70$ (Appendix~\ref{sec:leave_one_out}). The pooled association survives removal of any single deployment. Excluding GLM-5.3, the largest contributor, reduces the pooled $z$ from 5.19 to 2.24 ($p = 0.025$). Both 100\%-accuracy controls return the required boundary value: $P(C{=}1)=1$ makes the gap exactly zero.

Excluding the only batch requiring deduplication, which contributes 100 of the 1000 rows, leaves $n = 900$ rows with no deduplication. The gap remains $+7.26$pp ($\chi^2 = 20.54$, $p < 0.0001$), so the result does not depend on that batch.

\paragraph{Operation count accounts for part of the gap.}
Conditioning on $C=1$ selects instances with fewer interventionable operations, where compliance is higher. The covariance decomposition \eqref{eq:total_cov} assigns 42.8\% (GLM-5.3) and 21.9\% (flash) of the gap to variation between operation-count strata (Table~\ref{tab:decomp}). The GLM-5.3 estimate is more precise ($n=1000$); flash's smaller strata ($n=300$) yield a bootstrap interval spanning zero. The decomposition depends on the covariate: oracle-displacement magnitude accounts for $26.6\%$ on the same GLM-5.3 records. The 42.8\% share therefore measures an association with operation count and does not identify a causal effect of difficulty. Inverse-propensity weighting on that count removes 23--34\% of the raw gap across two propensity specifications, leaving most of the gap after either adjustment.

The trace-matching detector records $0\%$ shortcut use under neutral framing and $10\%$ in GLM-5.3 under adversarial elicitation. After excluding 9 null rows, the complete-case cue-correctness contrast ($n=291$) is $-1.6$pp ($z = -0.24$, $p = 0.81$), based on 112/208 correct-cue rows and 46/83 incorrect-cue rows. We detect no difference in this contrast, and the detector cannot identify unacknowledged use of a correct cue (Appendix~\ref{sec:discussion_full}). Appendix~\ref{sec:shortcut_elicitation} reports targeted explicit-reliance probes on separate deployments. The headline runs do not test the collider predictions of Section~\ref{sec:collider_case}; their scope is specified in Appendix~\ref{sec:limitations_full}.

Figure~\ref{fig:proxy_decomposition} shows the GLM-5.3 decomposition computed from the joint counts in Table~\ref{tab:depth_strat}. Of the $7.358$pp total gap, $3.146$pp lies between operation-count strata and $4.212$pp lies within them. The within-stratum component remains unattributed. The between-stratum share has a descriptive bootstrap interval of $[26.8,69.1]\%$ (Table~\ref{tab:decomp}); the plot shows the point estimate.

\begin{figure}[!htbp]
\centering
\includegraphics[width=\textwidth]{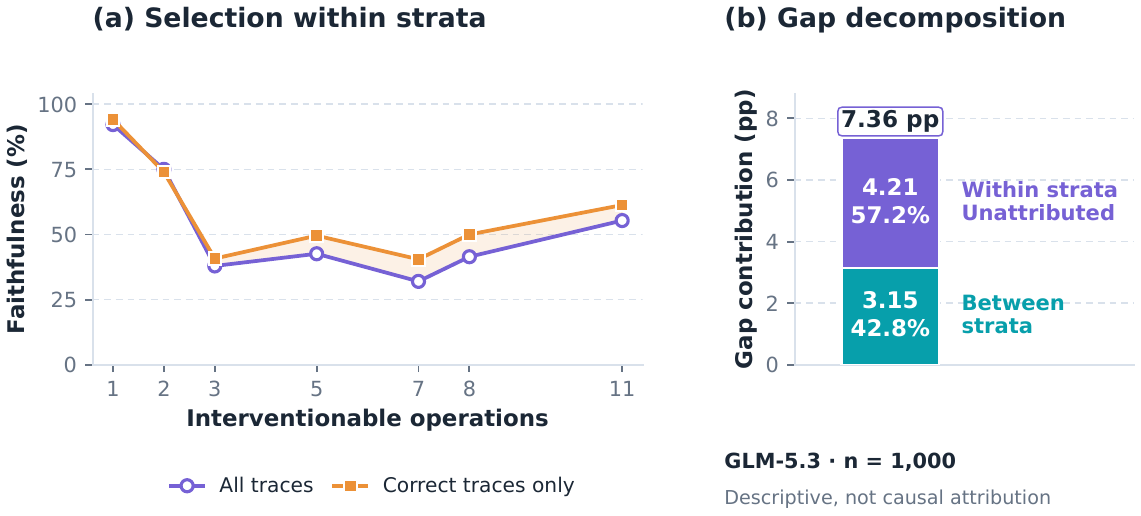}
\caption{\textbf{Operation-count decomposition.} GLM-5.3 faithfulness across strata and the corresponding covariance split. The split is descriptive; the within-stratum component remains unattributed.}
\label{fig:proxy_decomposition}
\end{figure}

\paragraph{The gap persists after adjustment and response filtering.}
Reweighting reduces the selected estimate but leaves it above the direct all-trace estimate (Figure~\ref{fig:robustness_checks}). Dropping null responses also reduces every deployment's gap, with all four point estimates remaining positive. Appendix~\ref{sec:ipw_details} specifies the propensity models, and Appendix~\ref{sec:compliance_discussion} gives the null-handling counts.

\begin{figure}[!htbp]
\centering
\includegraphics[width=\textwidth]{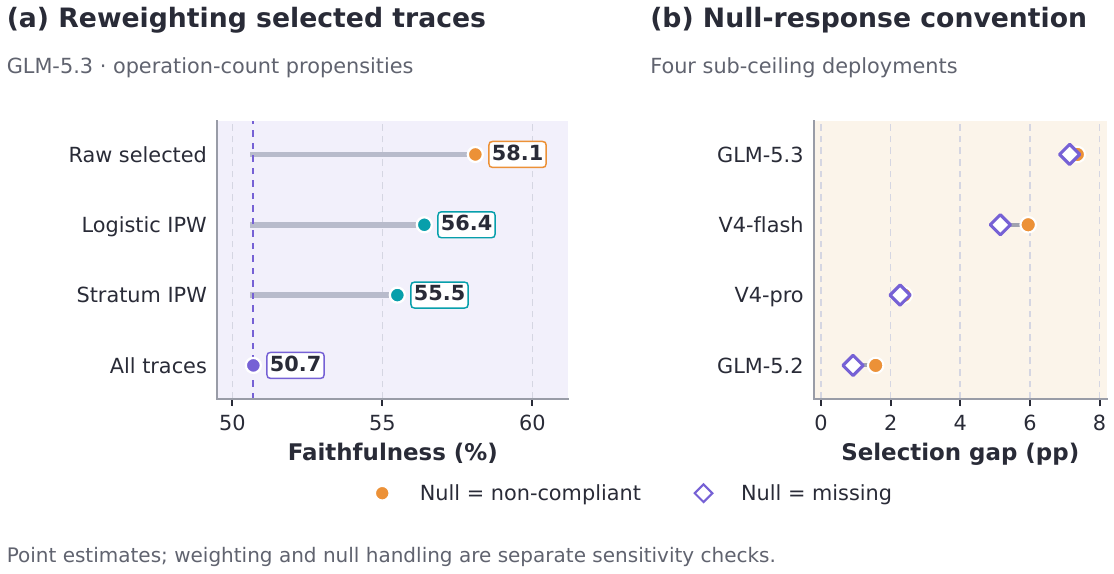}
\caption{\textbf{Sensitivity to adjustment and response handling.} Weighted selected estimates are compared with all traces; deployment gaps are compared under two null-response conventions. Markers show point estimates.}
\label{fig:robustness_checks}
\end{figure}

Table~\ref{tab:empirical} reports the marginal $\Delta = \Fsel - \Fpop$, which includes selected-weighted within-stratum gaps and a between-stratum composition term. The theory identifies the stratum-specific $\Delta^{(d)}$ (Appendix~\ref{sec:estimand}). For GLM-5.3, the marginal gap is $+7.4$pp, while selected-weighted difficulty standardization, $\Delta_{\mathrm{std},C} \coloneqq \sum_d P(D{=}d \mid C{=}1)\,(\Fsel^{(d)} - \Fpop^{(d)})$, gives $+4.2$pp; the naive average gives $+4.81$pp. Using the unrounded joint counts, the between-stratum term of \eqref{eq:total_cov} is approximately $3.15$pp, separate from the within-stratum effect.

\paragraph{The shortcut-mixture special case does not fit these runs.}
The strict-reliance mixture in \S\ref{sec:collider_case} overpredicts the measured gaps. Under its active-shortcut assumption, solving each deployment's $\Fpop$ and $P(C{=}1)$ for an implied $p_s$ under Corollary~\ref{cor:simplified} gives $+0.625$, $+0.247$, $+0.016$ and $-0.092$ for glm-5.2, v4-pro, GLM-5.3 and flash. Flash's negative value falls outside the probability range. Substitution gives gaps of $11.7$--$56.3$pp against $1.6$--$7.4$pp observed, a factor of $6.4$--$11.7$ (Table~\ref{tab:overprediction}). These discrepancies support estimating the observed $(A,C)$ gap and its operation-count decomposition directly, without imposing the mixture assumptions.

Table~\ref{tab:overprediction} reports each deployment's $\Fpop$ and $P(C=1)$, the implied $p_s$ under Corollary~\ref{cor:simplified}, and the corresponding prediction. We retain the flash row to document the incompatibility: its negative implied shortcut prevalence makes that prediction invalid as a probability-model result.

\begin{table}[t]
\caption{\textbf{Observed and predicted selection gaps.} Predictions use the shortcut-mixture special case; the flash prediction lies outside the valid probability regime.}
\label{tab:overprediction}
\centering
\fontsize{9}{10.8}\selectfont
\setlength{\tabcolsep}{4pt}
\renewcommand{\arraystretch}{1.18}
\arrayrulecolor{black!35}
\setlength{\heavyrulewidth}{0.65pt}
\setlength{\lightrulewidth}{0.35pt}
\begin{tabularx}{\linewidth}{@{}>{\raggedright\arraybackslash}Xcccccc@{}}
\toprule
\rowcolor[HTML]{F1F6FE}
\textbf{\boldmath Model} & \textbf{\boldmath $\Fpop$} & \textbf{\boldmath Acc.} & \textbf{\boldmath Implied $p_s$} & \textbf{\boldmath Predicted} & \textbf{\boldmath Observed} & \textbf{\boldmath Ratio} \\
\rowcolor[HTML]{F1F6FE}
 & \textbf{\boldmath ($= p_r$)} & \textbf{\boldmath ($=P(C{=}1)$)} &  & \textbf{\boldmath $\Delta$ (pp)} & \textbf{\boldmath $\Delta$ (pp)} &  \\

\midrule
glm-5.2 & \cellcolor[HTML]{F5F3FC}0.440 & 79.0\% & $+0.625$ & \cellcolor[HTML]{FFF0DC}\color[HTML]{1B2735}11.7 & \cellcolor[HTML]{FFF0DC}\color[HTML]{1B2735}1.6 & $7.3\times$ \\
deepseek-v4-pro & \cellcolor[HTML]{F5F3FC}0.555 & 66.5\% & $+0.247$ & \cellcolor[HTML]{FFF0DC}\color[HTML]{1B2735}28.0 & \cellcolor[HTML]{FFF0DC}\color[HTML]{1B2735}2.4 & $11.7\times$ \\
\rowcolor[HTML]{EAF1FC}
GLM-5.3 & \cellcolor[HTML]{F5F3FC}0.507 & 51.5\% & $+0.016$ & \cellcolor[HTML]{FFF0DC}\color[HTML]{1B2735}47.7 & \cellcolor[HTML]{FFF0DC}\color[HTML]{1B2735}7.4 & $6.4\times$ \\
deepseek-v4-flash & \cellcolor[HTML]{F5F3FC}0.527 & 48.3\% & $-0.092$ & \cellcolor[HTML]{FFF0DC}\color[HTML]{1B2735}56.3 & \cellcolor[HTML]{FFF0DC}\color[HTML]{1B2735}6.0 & $9.4\times$ \\
\bottomrule
\end{tabularx}
\end{table}
% END sections/empirical
% BEGIN sections/conclusion
\section{Conclusion}
\label{sec:conclusion}

We show that correctness conditioning changes the population used to estimate CoT faithfulness and quantify the resulting gap through $\Delta = \Cov(A,C)/P(C{=}1)$. Our protocol review separates this restriction from other selection rules used in prior work. On arithmetic DAGs, the primary GLM-5.3 run has a $+7.4$pp gap ($n=1000$, $p<0.0001$); operation-count variation accounts for 42.8\%, and the remainder is unattributed. Two of the four sub-ceiling deployments have individually significant gaps. These results establish selection bias in the measured setting. No shortcut use was detected in the neutral-framing runs, so these runs do not support attributing the observed bias to the collider pathway. Faithfulness evaluations should state their reporting population and use all traces whenever counterfactual compliance is observable for every instance.
% END sections/conclusion

% BEGIN sections/statements
\section*{Reproducibility Statement}

The procedural task generators use deterministic seed streams, and the symbolic oracle computes exact outputs. The main text and appendices specify the simulation parameters, estimands, prompts, and scoring conventions. Headline model calls use temperature $=0$ (greedy decoding). When a serving endpoint returned different responses at that setting, we retained the first completed response for the reported run. We conducted the headline experiments in August 2026 and the exploratory framing probe in September 2026. Each model--seed--instance is counted once.

\section*{Data and Code Availability}

The website source, project film, paper figures, PDF, and release checks are available at \url{https://github.com/CCS-CoTFE/CCS-CoTFE}. The project website is \url{https://ccs-cotfe.github.io/CCS-CoTFE/}. The paper source bundle contains the LaTeX files and figure assets needed to compile the manuscript.

The release contains the aggregate results reported in the tables and figures. Raw API logs, model completions, provider request identifiers, intermediate exports, the working literature-audit ledger, and credential-bound generators are not included. Exact row-level reanalysis requires a separate anonymized data release.

\section*{Ethics Statement}

We study the statistical properties of evaluation methods without human subjects, private data, or model training. All language model API calls use publicly available commercial endpoints under standard terms of service. The procedural tasks are synthetic and contain no sensitive content.
% END sections/statements

\appendix
% BEGIN sections/reproducibility
\section{Experimental Configuration}
\label{sec:expconfig}

Task families and experimental phases use different sample sizes, which we report separately. The implementation check uses small deterministic batches across three synthetic task families. The real-model study uses arithmetic DAGs with a primary run of $n=1000$ and additional sub-ceiling deployments of $n=300$, $n=200$, and $n=100$, together with two $100\%$-accuracy boundary controls. Each model--seed--instance key is counted once before aggregation.

The arithmetic DAG generator uses a depth range of $[2,4]$, a value range of $[1,20]$, and a planted-cue validity control of $q=0.7$. This control differs from the theoretical shortcut prevalence $p_s$ used in the implementation check.

The prompt template for the real-model runs is:

\begin{quote}
\ttfamily Solve this step by step. Show your work. \par\vspace{2pt}
\{task prompt\} \par\vspace{2pt}
Your final answer (just the number):
\end{quote}

The counterfactual condition replaces the intervened step line with the new value (e.g.\ \texttt{S3 = 12 (given)}) and keeps the rest of the task unchanged. The empirical runs use no system-level objective or persona prompt. Persona-modulated agents in simulation (Section~\ref{sec:simulation}) are stochastic compliance models with a fixed random seed: for each instance, the agent answers correctly with probability $\pi_{rs}$ and complies with intervention with probability $\eta_{rs}$, drawn from the closed-form cell probabilities. These agents check the mathematics and implementation; they do not model the behavior of a real language model.
% END sections/reproducibility
% BEGIN sections/appendix
\section{Implementation Consistency Check}
\label{sec:simulation}

We verify the numerical implementation by sampling agents from the derived cell probabilities. The simulator imposes the model assumptions, so this experiment checks the implementation and does not independently test the theory. Sweeping $(p_r, p_s) \in [0,1]^2$ gives a mean absolute error of 0.006 against Corollary~\ref{cor:simplified}. The implementation reproduces the monotone decrease in $p_s$ (Corollary~\ref{cor:ps_monotone}) and the peak at $p_r^* \approx 0.456$ (Figure~\ref{fig:bias_surface}, Appendix~\ref{sec:sim_families}).

The three task families differ in dependency structure and reliance level. Table~\ref{tab:simulation} compares them at a common theoretical shortcut prevalence $p_s = 0.7$. The gaps agree to within $0.6$pp (mean $+8.2$pp, range $+7.9$ to $+8.5$pp), while per-family reliance $p_r$ ranges from $0.33$ to $0.58$. This checks the closed-form implementation across integer arithmetic, Boolean logic, and graph reachability. The headline real-model evidence uses arithmetic tasks.

\begin{table}[htbp]
\caption{\textbf{Implementation consistency.} Three procedural task families evaluated under a common shortcut prevalence.}
\label{tab:simulation}
\centering
\fontsize{9}{10.8}\selectfont
\setlength{\tabcolsep}{4pt}
\renewcommand{\arraystretch}{1.18}
\arrayrulecolor{black!35}
\setlength{\heavyrulewidth}{0.65pt}
\setlength{\lightrulewidth}{0.35pt}
\begin{tabularx}{\linewidth}{@{}>{\raggedright\arraybackslash}Xccccc@{}}
\toprule
\rowcolor[HTML]{F1F6FE}
\textbf{\boldmath Family} & \textbf{\boldmath $n$} & \textbf{\boldmath $p_r$} & \textbf{\boldmath $\Fpop$} & \textbf{\boldmath $\Fsel$} & \textbf{\boldmath Gap (pp)} \\

\midrule
Arithmetic DAG & 30 & 0.33 & \cellcolor[HTML]{F5F3FC}0.333 & \cellcolor[HTML]{FFF6E9}0.417 & \cellcolor[HTML]{FFF0DC}\color[HTML]{1B2735}+8.3 \\
Boolean circuit & 50 & 0.58 & \cellcolor[HTML]{F5F3FC}0.580 & \cellcolor[HTML]{FFF6E9}0.659 & \cellcolor[HTML]{FFF0DC}\color[HTML]{1B2735}+7.9 \\
Graph reachability & 50 & 0.52 & \cellcolor[HTML]{F5F3FC}0.520 & \cellcolor[HTML]{FFF6E9}0.605 & \cellcolor[HTML]{FFF0DC}\color[HTML]{1B2735}+8.5 \\
\midrule
\rowcolor[HTML]{EAF1FC}
\textbf{Mean} & & & & & \cellcolor[HTML]{FFF0DC}\color[HTML]{1B2735}\textbf{+8.2} \\
\bottomrule
\end{tabularx}
\end{table}

Under the correctness parameters specified in Appendix~\ref{sec:reversal_sensitivity}, Model $\mathcal{A}$ ($p_r = 0.45$, $p_s = 0.90$) has $\Fpop = 0.450$ but $\Fsel = 0.504$, while Model $\mathcal{B}$ ($p_r = 0.70$, $p_s = 0.01$) has $\Fpop = 0.700$ but $\Fsel = 0.308$: conditioning reverses the ranking, with 40.6\% of the 39,999 full-grid candidates reversing at $\pi_{10}=0.001$. These conditions underlie the prospective analysis in Section~\ref{sec:collider_case}. This example requires an active shortcut branch, which we did not detect in our real-model runs under neutral framing.

\section{Simulation: Cross-Family Detail}
\label{sec:sim_families}

Table~\ref{tab:simulation} (Section~\ref{sec:simulation}) reports the per-family simulation results. The theoretical shortcut prevalence is $p_s = 0.7$ for every row; this is distinct from the planted-cue validity control $q$ in Section~\ref{sec:tasks}. The per-family CoT-reliance level $p_r$ differs and equals $\Fpop$ under strict reliance. Boolean circuit and graph tasks use $n=50$ for richer structural variation; arithmetic DAGs use $n=30$ due to simpler compositional structure. The comparison varies the dependency structure while retaining the intervention protocol and oracle design of Section~\ref{sec:method}.

\begin{figure}[htbp]
\centering
\includegraphics[width=1.00\textwidth]{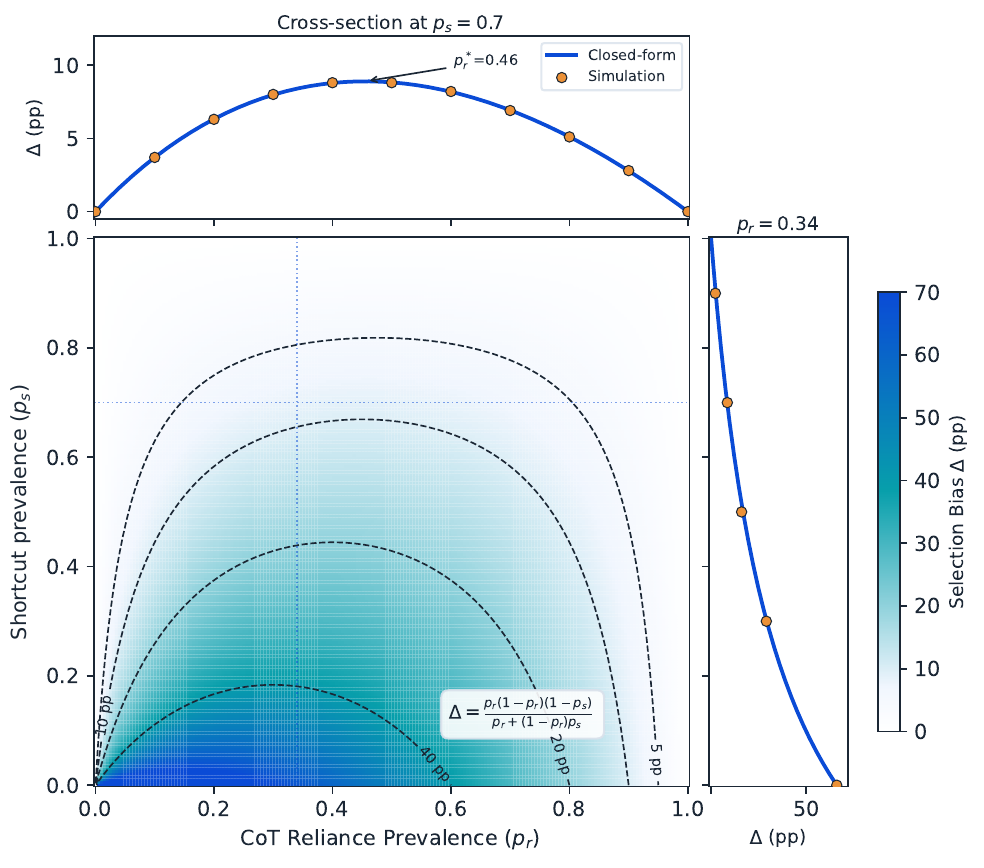}
\caption{\textbf{Selection bias surface.} Closed-form predictions with simulated points overlaid on the marginal slices.}
\label{fig:bias_surface}
\end{figure}

\section{Estimand: Within-Stratum vs.\ Marginal}
\label{sec:estimand}

We distinguish stratum-specific and marginal estimands. The theory's Propositions~\ref{prop:collider_bias}, \ref{prop:general_gap} and \ref{prop:reversal} are stated at fixed difficulty $D = d$ and identify the \emph{stratum-specific} $\Delta^{(d)}$. Conditioning on $d$ permits the use of $R \perp S \mid D$ and the $T$-marginalization. It also removes the difficulty-confounding pathway from $\Delta^{(d)}$, since $D$ no longer varies. Proposition~\ref{prop:difficulty} is \emph{marginal} in $D$ and applies to real models, where latent $R$ and $S$ cannot be conditioned on.

Table~\ref{tab:empirical} reports the \emph{marginal} $\Delta = \Fsel - \Fpop$. Its within-stratum component averages the $\Delta^{(d)}$ under the difficulty distribution among selected instances; a separate composition term accounts for the change in that distribution caused by correctness conditioning. The marginal gap is therefore not identified by the within-stratum result alone. We interpret the real-model gaps as consistent with difficulty confounding, without applying the shortcut-mixture special case to them.

The within-stratum analysis (Appendix~\ref{sec:depth_analysis}) decomposes the marginal gap by operation count. Gaps persist within strata, beyond the between-stratum variation. This decomposition leaves the collider contribution unidentified.

\paragraph{Magnitude reconciliation.} The shortcut-mixture special case is inconsistent with the observed deployments in two ways. First, setting $p_s = 0$ to match the measured $0\%$ shortcut-use rate reduces Corollary~\ref{cor:simplified} to $\Delta = 1 - P(C{=}1)$. This expression is monotone, has no interior maximum, and predicts ${\sim}48$pp at GLM-5.3's $52\%$ accuracy, compared with $+7.4$pp observed. Second, the implementation-check parameter $p_s = 0.7$ implies an accuracy floor of $0.70$, above the measured $48$--$52\%$ accuracy of these deployments.

Solving each model's own $\Fpop$ and accuracy for an implied $p_s$ gives $+0.625$, $+0.247$, $+0.016$, and $-0.092$ (glm-5.2, v4-pro, GLM-5.3, flash). The negative value is not a valid probability, and these inferred values do not establish the measured $0\%$ shortcut use. Table~\ref{tab:overprediction} (Section~\ref{sec:empirical}) reports the inputs needed to reproduce the predicted/observed ratios; Section~\ref{sec:empirical} discusses their disagreement.

We therefore use the mixture special cases for simulation and idealised design analysis, and analyse real-model gaps through the observed operation-count covariate.

The accuracy-axis peak occurs at $P^{*} = \sqrt{p_s}$, so with the implementation-check value $p_s = 0.7$ it lies at $83.7\%$: the four sub-ceiling deployments ($48$--$79\%$) lie below it, while the two ceiling controls ($100\%$) lie above. No tested deployment is near the peak. The ceiling controls cannot localise it because $P(C{=}1) = 1$ makes $\Delta = 0$ an algebraic identity. Probing a peak near $50\%$ accuracy would require $p_s \approx 0.25$, which is a different theoretical regime from the one used in the implementation check. The accuracy band in Figure~\ref{fig:goldilocks} is therefore descriptive, not a prediction.

\section{Proof Details}
\label{sec:proofs}

\subsection{Corollary~\ref{cor:accuracy_hump}: Hump on the Accuracy Axis}

With $r \coloneqq \sqrt{p_s}$ and strict reliance, Corollary~\ref{cor:hump} locates the hump at $p_r^* = r/(1+r)$ in CoT-reliance space. Substituting into Corollary~\ref{cor:simplified}: $p_r^*(1-p_r^*) = r/(1+r)^2$, and $p_r^* + (1-p_r^*)p_s = r/(1+r) + r^2(1+r)^{-1} = r$. Hence
\begin{equation}
\Delta(P^*) = \frac{r(1-p_s)}{(1+r)^2\,r} = \frac{1-r^2}{(1+r)^2} = \frac{1-r}{1+r}, \qquad P(p_r^*) = p_r^* + (1-p_r^*)p_s = r = \sqrt{p_s}.
\end{equation}
Since $P$ is affine and strictly increasing in $p_r$, the unique interior maximum in $p_r$ maps to the unique interior maximum in $P$. Attainable accuracy spans $[p_s, 1]$; the peak $P^* = \sqrt{p_s}$ lies in that interval for all $p_s \in (0,1)$.

Proofs of Propositions~\ref{prop:collider_bias} and~\ref{prop:difficulty} are given in the main text. The collider pathway derivations (Propositions~\ref{prop:general_gap} and~\ref{prop:reversal}, Corollaries~\ref{cor:simplified}--\ref{cor:accuracy_hump}) are in Appendix~\ref{sec:collider_full}. This appendix contains additional derivation details.

\section{Collider Pathway: Full Derivations}
\label{sec:collider_full}

This section derives the collider results summarized in Section~\ref{sec:collider_case} for a population with an active shortcut branch ($S > 0$).

The strict-reliance special case sets $\eta_{00} = \eta_{01} = 0$ and $\eta_{10} = \eta_{11} = 1$: the model's answer changes under intervention if and only if it genuinely relies on the CoT. Then $\Fpop = p_r$. Define the stratum-specific correctness rates
\begin{align}
g_1(p_s) &\coloneqq P(C{=}1 \mid R{=}1) = (1{-}p_s)\,\pi_{10} + p_s\,\pi_{11}, \label{eq:g1} \\
g_0(p_s) &\coloneqq P(C{=}1 \mid R{=}0) = (1{-}p_s)\,\pi_{00} + p_s\,\pi_{01}. \label{eq:g0}
\end{align}

\begin{proposition}[Strict-Reliance Gap]
\label{prop:general_gap}
Under the strict-reliance assumption, $\Fsel = P(R{=}1 \mid C{=}1)$, and the gap is:
\begin{equation}
\Delta = \frac{p_r(1 - p_r)(g_1 - g_0)}{p_r \, g_1 + (1 - p_r) \, g_0}.
\label{eq:delta_general}
\end{equation}
The sign of $\Delta$ is determined by $g_1 - g_0$: correctness conditioning overestimates faithfulness when CoT reliance makes correctness more likely ($g_1 > g_0$), and underestimates it otherwise.
\end{proposition}

\begin{proof}
Under strict reliance, $\Fsel = P(R{=}1 \mid C{=}1) = p_r g_1 / [p_r g_1 + (1{-}p_r) g_0]$ by Bayes' rule. Subtracting $\Fpop = p_r$ and simplifying yields~\eqref{eq:delta_general}.
\end{proof}

\begin{corollary}[Simplified Gap Under Perfect Faithful Reasoning]
\label{cor:simplified}
When faithful reasoning is always correct ($\pi_{10} = \pi_{11} = 1$), shortcut-only correctness is certain ($\pi_{01}=1$), and the neither-route cell is always wrong ($\pi_{00} = 0$), we have $g_1 = 1$ and $g_0 = p_s$, giving:
\begin{equation}
\Delta = \frac{p_r(1 - p_r)(1 - p_s)}{p_r + (1 - p_r) \, p_s}.
\label{eq:delta_simplified}
\end{equation}
\end{corollary}

This simplified form is a theoretical mixture for an active shortcut branch. The planted-cue generator in Section~\ref{sec:tasks} uses the separate validity control $q$ and does not estimate $p_s$ in the real-model runs. For multiple-choice tasks with $k$ options, $\pi_{00} = 1/k > 0$ and the general formula~\eqref{eq:delta_general} applies; the simplified form remains a close approximation when $1/k \ll p_s$.

\begin{corollary}[Monotone Decrease in Shortcut Prevalence]
\label{cor:ps_monotone}
For fixed $p_r \in (0, 1)$, $\partial \Delta / \partial p_s < 0$: the gap~\eqref{eq:delta_simplified} is strictly decreasing in shortcut prevalence.
\end{corollary}

\begin{proof}
Write $\Delta(p_s) = a\,b\,(1 - p_s) / (a + b\,p_s)$ with $a = p_r > 0$ and $b = 1{-}p_r > 0$. Differentiating:
$\partial \Delta / \partial p_s = -a\,b\,(a + b) / (a + b\,p_s)^2 < 0$.
\end{proof}

\begin{corollary}[Hump-Shaped Prevalence Response]
\label{cor:hump}
For fixed $p_s \in (0, 1)$, $\Delta(p_r)$ from~\eqref{eq:delta_simplified} is zero at $p_r = 0$ and $p_r = 1$, with a unique interior maximum at $p_r^* = \sqrt{p_s}/(1 + \sqrt{p_s})$.
\end{corollary}

\begin{proof}
Setting $d\Delta/dp_r = 0$: the first-order condition yields $p_r^2 = p_s(1{-}p_r)^2$, hence $p_r/(1{-}p_r) = \sqrt{p_s}$, giving $p_r^* = \sqrt{p_s}/(1 + \sqrt{p_s})$. The boundary values $\Delta(0) = \Delta(1) = 0$ confirm this is a maximum.
\end{proof}

\begin{corollary}[Hump on the Accuracy Axis]
\label{cor:accuracy_hump}
Under the assumptions of Corollary~\ref{cor:simplified}, with $p_s\in(0,1)$ and $P \coloneqq P(C{=}1)$, the bias has a unique interior maximum in accuracy at $P^* = \sqrt{p_s}$, where $\Delta(P^*) = (1-\sqrt{p_s})/(1+\sqrt{p_s})$ (proof in Appendix~\ref{sec:proofs}). Accuracy below $P^*$ lies on the rising limb, above it on the falling limb.
\end{corollary}

\begin{proposition}[Ranking Reversal]
\label{prop:reversal}
Consider two models $\mathcal{A}$ and $\mathcal{B}$ in the strict-reliance case with shared correctness parameters $(\pi_{rs})$ and positive selection probabilities $q_m=p_{r,m}g_{1,m}+(1-p_{r,m})g_{0,m}>0$ for $m\in\{A,B\}$. If $p_{r,A} < p_{r,B}$, a ranking reversal $\Fsel(\mathcal{A}) > \Fsel(\mathcal{B})$ occurs if and only if
\begin{equation}
p_{r,A}(1-p_{r,B})g_{1,A}g_{0,B}
>p_{r,B}(1-p_{r,A})g_{1,B}g_{0,A}.
\label{eq:reversal_condition}
\end{equation}
For $0<p_{r,A}<p_{r,B}<1$ and positive $g_{0,m},g_{1,m}$, this is equivalent to $L_A/L_B>\mathrm{odds}(p_{r,B})/\mathrm{odds}(p_{r,A})$, where $L_m=g_{1,m}/g_{0,m}$ and $\mathrm{odds}(p)=p/(1-p)$. The cross-product form also covers zero cell probabilities whenever $q_A,q_B>0$.
\end{proposition}

When the less CoT-reliant model has greater shortcut prevalence, reversal requires the complementarity condition $\pi_{11} \pi_{00} > \pi_{10} \pi_{01}$, plausible for hard tasks where reliance without shortcut support often produces invalid traces ($\pi_{10} \approx 0$) but reliance plus shortcut support yields correct answers ($\pi_{11} \approx 1$). This requirement concerns $p_{s,A}>p_{s,B}$; when $p_{s,A}<p_{s,B}$, the opposite sign can permit reversal under Proposition~\ref{prop:reversal}.

Under the special case in Corollary~\ref{cor:simplified}, the bias is largest at intermediate CoT reliance ($p_r^* = \sqrt{p_s}/(1+\sqrt{p_s})$) and, on the accuracy axis, at $P^* = \sqrt{p_s}$. A peak near 50\% accuracy requires $p_s \approx 0.25$. At $p_s = 0$ this special case reduces to $\Delta = 1 - P$, which is monotone and has no interior maximum. Its interior peak therefore requires a non-degenerate shortcut branch; testing it requires deployments in that regime.

\subsection{Ranking Reversal Threshold Derivation}

For two models $\mathcal{A}$ and $\mathcal{B}$ with $p_{r,A} < p_{r,B}$ and shared correctness parameters, the reversal condition $\Fsel(\mathcal{A}) > \Fsel(\mathcal{B})$ expands to:
\begin{equation}
\frac{p_{r,A} \, g_{1,A}}{p_{r,A} \, g_{1,A} + (1{-}p_{r,A}) \, g_{0,A}} > \frac{p_{r,B} \, g_{1,B}}{p_{r,B} \, g_{1,B} + (1{-}p_{r,B}) \, g_{0,B}},
\end{equation}
Multiplying by the positive denominators and cancelling the common term gives~\eqref{eq:reversal_condition}. Dividing by positive $p_{r,A}(1-p_{r,B})g_{0,A}g_{1,B}$ gives the likelihood-ratio form under the additional positivity conditions in Proposition~\ref{prop:reversal}.

\section{Measurement Validation}
\label{sec:measurement}

The faithful baseline ($n = 50$) gives $\Fpop = \Fsel = 1.000$, gap $= 0.000$, and shortcut rate $= 0\%$. In this check, the oracle introduces no false non-compliance and the interventions correctly modify downstream values.

\section{Protocol Review Details}
\label{sec:audit_full}

Table~\ref{tab:audit_full} records the procedure and reporting population found in each inspected source. Years follow the bibliography; the final column links to the pinned arXiv version and gives its section, equation, or table. A row describes the named analysis, not every experiment in that paper. In particular, an accuracy endpoint, a process-validity judgement, and intervention compliance are different measurements.

\begin{table}[!htbp]
\caption{\textbf{Reporting populations in prior work.} Protocol-specific evidence from 18 papers.}
\label{tab:audit_full}
\centering
\fontsize{8.5}{9.7}\selectfont
\setlength{\tabcolsep}{3pt}
\renewcommand{\arraystretch}{1.15}
\arrayrulecolor{black!35}
\setlength{\heavyrulewidth}{0.65pt}
\setlength{\lightrulewidth}{0.35pt}
\begin{tabularx}{\linewidth}{@{}>{\raggedright\arraybackslash}p{0.18\linewidth}c>{\raggedright\arraybackslash}p{0.20\linewidth}>{\raggedright\arraybackslash}X>{\raggedright\arraybackslash}p{0.12\linewidth}@{}}
\toprule
\rowcolor[HTML]{F1F6FE}
\textbf{Paper} & \textbf{Year} & \textbf{Procedure} & \textbf{Reporting population} & \textbf{Source} \\
\midrule
\rowcolor[HTML]{F7F9FC}
\citeauthor{al2023measuring} & 2023 & Trace truncation, mistakes, paraphrase & Collected traces; length-weighted aggregation & \href{https://arxiv.org/abs/2307.13702v1}{v1}; \S2.1--2.6 \\
\citeauthor{al2023language} & 2023 & Biased input cues & Main input-bias tests; correctness-selected qualitative examples & \href{https://arxiv.org/abs/2305.04388v2}{v2}; \S3.1; Tab.~7 \\
\rowcolor[HTML]{F7F9FC}
\citeauthor{al2023faithful} & 2023 & Deterministic solver & Test-set accuracy; correct outputs for human plausibility & \href{https://arxiv.org/abs/2301.13379v3}{v3}; \S4.2, 6.3 \\
\citeauthor{al2023question} & 2023 & Truncation and corruption & All reasoning samples; gold-label subsets for bias tests & \href{https://arxiv.org/abs/2307.11768v2}{v2}; \S3.2--3.3 \\
\rowcolor[HTML]{F7F9FC}
\citeauthor{al2023knowledge} & 2023 & Retrieval and trace construction & Test-set task success; correct-answer training pool & \href{https://arxiv.org/abs/2308.13259v2}{v2}; Alg.~1; Tab.~1 \\
\citeauthor{al2024chain} & 2024 & With/without-CoT agreement & Dataset average; no answer-correctness screen in metric & \href{https://arxiv.org/abs/2402.14897v3}{v3}; \S3, Eq.~1 \\
\rowcolor[HTML]{F7F9FC}
\citeauthor{al2024making} & 2024 & Counterfactual mediation & Mean test-set accuracy differences & \href{https://arxiv.org/abs/2402.13950v4}{v4}; \S4; Tab.~13 \\
\citeauthor{al2024hardness} & 2024 & Early-answer consistency & Test CoTs; correct-only training/examples are separate & \href{https://arxiv.org/abs/2406.10625v2}{v2}; \S2--4 \\
\rowcolor[HTML]{F7F9FC}
\citeauthor{arcuschin2025wild} & 2025 & Input pairs; shortcut detection & All comparison pairs; correct Putnam responses & \href{https://arxiv.org/abs/2503.08679v6}{v6}; \S2.1, 3.1 \\
\citeauthor{tutek2025unlearning} & 2025 & Unlearning; mistake insertion & Agreement between CoT and no-CoT predictions & \href{https://arxiv.org/abs/2502.14829v4}{v4}; Tab.~2 \\
\rowcolor[HTML]{F7F9FC}
\citeauthor{chen2025reasoning} & 2025 & Injected hints & Pairs where the answer switches to the hint & \href{https://arxiv.org/abs/2505.05410v1}{v1}; \S2.1 \\
\citeauthor{bogdan2025anchors} & 2025 & Step resampling and attention & Outcome-balanced MATH traces; correct MMLU CoTs & \href{https://arxiv.org/abs/2506.19143v4}{v4}; \S2.1, 6.1 \\
\rowcolor[HTML]{F7F9FC}
\citeauthor{zaman2025causal} & 2025 & Knowledge edits; corruption & Constructed explanation pairs; top-class score change & \href{https://arxiv.org/abs/2502.18848v3}{v3}; \S3.2; App.~D.1 \\
\citeauthor{al2025faithcot} & 2026 & Instance-level detection & Annotated correct and incorrect traces & \href{https://arxiv.org/abs/2510.04040v2}{v2}; \S3.2; App.~B.1 \\
\rowcolor[HTML]{F7F9FC}
\citeauthor{al2026faithfulness} & 2026 & Information flow; hint edits & Wrong-hint cases; behavior-selected verbalization & \href{https://arxiv.org/abs/2605.24286v1}{v1}; \S6.1--6.4 \\
\citeauthor{al2026counterfactual} & 2026 & Counterfactual input simulation & Simulator and cue-influence test sets & \href{https://arxiv.org/abs/2602.20710v2}{v2}; \S3--4 \\
\rowcolor[HTML]{F7F9FC}
\citeauthor{bonafide2026} & 2026 & Ground-truth meta-evaluation & Correct outright or wrong-hint-following responses & \href{https://arxiv.org/abs/2605.25052v1}{v1}; \S4.1--4.2 \\
\citeauthor{c2faith2026} & 2026 & Step-validity/coverage judges & Perfect source chains: every step labelled correct & \href{https://arxiv.org/abs/2603.05167v2}{v2}; \S3; Alg.~1 \\
\bottomrule
\end{tabularx}
\par\vspace{3pt}\noindent\parbox{\linewidth}{\raggedright\footnotesize Source locators refer to the linked versions and the named analyses.}
\end{table}

\subsection{Scope and Interpretation}
\label{sec:audit_rule}

We retain adjacent protocols to make their selection rules explicit. Lyu et al. derive the answer by deterministic execution of the symbolic trace; their human plausibility study selects correct outputs~\citep{al2023faithful}. Tanneru et al. select correct demonstrations or fine-tuning examples in some variants, which does not establish a correct-only test denominator~\citep{al2024hardness}. Counterfactual simulation training intervenes on inputs and trains the task model using simulator feedback~\citep{al2026counterfactual}; it is not the same operation as editing a task model's recorded trace. These distinctions prevent training, input, and test-population choices from being counted as one intervention protocol.

The reviewed sources also include protocols with correct and incorrect outputs. Radhakrishnan et al. average their trace-manipulation metrics over all reasoning samples~\citep{al2023question}; Paul et al. report mean test-set accuracy differences under counterfactual mediation~\citep{al2024making}. FaithCoT-Bench evaluates detectors on annotated correct and incorrect traces~\citep{al2025faithcot}. In Zaman and Srivastava's constructed explanation-pair testbed, corruption changes the score of the model's top predicted class~\citep{zaman2025causal}. These denominators differ from post-hoc selection on factual answer correctness.

This review identifies the reporting population for each measure in the cited versions. It compares reported protocols; estimating field-wide prevalence requires systematic sampling.

\section{Power Analysis}
\label{sec:power}

Table~\ref{tab:power} gives simulation-based power projections at $\alpha = 0.05$, computed from the initial $n = 100$ pilot before the scaled experiments. Both scaled runs subsequently reached the stated significance threshold. GLM-5.3 at $n = 1000$ achieved $p < 0.0001$ (Table~\ref{tab:empirical}); the projections remain based on the pilot effect size.

\begin{table}[htbp]
\caption{\textbf{Power projections.} Pilot-based estimates for the scaled deployment runs.}
\label{tab:power}
\centering
\fontsize{9}{10.8}\selectfont
\setlength{\tabcolsep}{4pt}
\renewcommand{\arraystretch}{1.18}
\arrayrulecolor{black!35}
\setlength{\heavyrulewidth}{0.65pt}
\setlength{\lightrulewidth}{0.35pt}
\begin{tabularx}{\linewidth}{@{}>{\raggedright\arraybackslash}Xcccc@{}}
\toprule
\rowcolor[HTML]{F1F6FE}
\textbf{\boldmath Model} & \textbf{\boldmath $n$ (pilot)} & \textbf{\boldmath Pilot $\Delta$ (pp)} & \textbf{\boldmath Power at $n$} & \textbf{\boldmath $n$ for 80\% power} \\

\midrule
glm-5.2 & 100 & \cellcolor[HTML]{FFF0DC}\color[HTML]{1B2735}+1.6 & 0.06 & $>$5{,}000 \\
\rowcolor[HTML]{EAF1FC}
glm-5.3 & 100 & \cellcolor[HTML]{FFF0DC}\color[HTML]{1B2735}+8.8 & 0.25 & $\approx$400 \\
deepseek-v4-flash & 100 & \cellcolor[HTML]{FFF0DC}\color[HTML]{1B2735}+3.8 & 0.09 & $\approx$1{,}800 \\
qwen3.8-max & 20 & \cellcolor[HTML]{F5F6F7}0.0 & \cellcolor[HTML]{F5F6F7}--- & \cellcolor[HTML]{F5F6F7}--- \\
\midrule
\rowcolor[HTML]{EAF1FC}
\multicolumn{5}{l}{\emph{Achieved:} glm-5.3 at $n = 1000$: $\Delta = +7.4$pp, $p < 0.0001$.} \\
\rowcolor[HTML]{EAF1FC}
\multicolumn{5}{l}{\emph{Achieved:} deepseek-v4-flash at $n = 300$: $\Delta = +6.0$pp, $p = 0.046$.} \\
\bottomrule
\end{tabularx}
\end{table}

The pilot projections give low power at $n = 100$. Detecting GLM-5.3's +8.8pp pilot gap at 80\% power requires approximately $n = 400$ instances; glm-5.2's smaller +1.6pp gap requires thousands. At the pilot $n = 100$, GLM-5.3's 95\% CI is [$-$1.5, +19.1]pp. Holding the effect at +8.8pp in the projection gives CI [+1.5, +16.1]pp at $n \approx 200$ and CI [+2.85, +14.75]pp at $n = 300$. These projected intervals depend on the pilot effect size.

\section{Within-Stratum Difficulty Analysis}
\label{sec:depth_analysis}

We stratify by \emph{interventionable-operation count}, the number of editable operations in an instance. This covariate takes the values $\{1,2,3,5,7,8,11\}$ in the generated tasks. The generator never emits counts 4, 6, 9 or 10; the tables retain these empty strata with $n=0$ to make the sample accounting explicit. Operation count is a \emph{proxy} for difficulty. It is fixed for each instance across models and precedes the model's answer.

Table~\ref{tab:depth_strat} presents GLM-5.3 results ($n = 1000$) stratified on that count.

In Tables~\ref{tab:depth_strat}, \ref{tab:twobytwo}, and~\ref{tab:strata_flash}, $a$ denotes compliant-and-correct, $b$ correct-but-non-compliant, $c$ incorrect-but-compliant, and $d$ incorrect-and-non-compliant. Missing or unparseable responses count as non-compliant. Each stratum's four cells sum to its sample size; the GLM-5.3 column totals match Table~\ref{tab:twobytwo} and reproduce its operation-count decomposition.

\begin{table}[htbp]
\caption{\textbf{GLM-5.3 results by operation count.} Joint outcome counts and selection gaps, with empty strata retained for completeness.}
\label{tab:depth_strat}
\centering
\fontsize{9}{10.8}\selectfont
\setlength{\tabcolsep}{4pt}
\renewcommand{\arraystretch}{1.18}
\arrayrulecolor{black!35}
\setlength{\heavyrulewidth}{0.65pt}
\setlength{\lightrulewidth}{0.35pt}
\begin{tabularx}{\linewidth}{@{}>{\raggedright\arraybackslash}Xrcccccccc@{}}
\toprule
\rowcolor[HTML]{F1F6FE}
\textbf{\boldmath Steps} & \textbf{\boldmath $n$} & \textbf{\boldmath $a$} & \textbf{\boldmath $b$} & \textbf{\boldmath $c$} & \textbf{\boldmath $d$} & \textbf{\boldmath Acc.} & \textbf{\boldmath $\Fpop$} & \textbf{\boldmath $\Fsel$} & \textbf{\boldmath $\Delta$ (pp)} \\

\midrule
1 (simple)    & 115 & 81 & 5  & 25 & 4  & 75\% & \cellcolor[HTML]{F5F3FC}0.922 & \cellcolor[HTML]{FFF6E9}0.942 & \cellcolor[HTML]{FFF0DC}\color[HTML]{1B2735}$+2.0$ \\
2 (easy)      & 104 & 51 & 18 & 27 & 8  & 66\% & \cellcolor[HTML]{F5F3FC}0.750 & \cellcolor[HTML]{FFF6E9}0.739 & \cellcolor[HTML]{EDF8F9}\color[HTML]{076872}$-1.1$ \\
3 (moderate)  & 232 & 53 & 77 & 35 & 67 & 56\% & \cellcolor[HTML]{F5F3FC}0.379 & \cellcolor[HTML]{FFF6E9}0.408 & \cellcolor[HTML]{FFF0DC}\color[HTML]{1B2735}$+2.8$ \\
\rowcolor[HTML]{F5F6F7}
4             & 0   & 0  & 0  & 0  & 0  & ---  & ---   & ---   & ---    \\
5 (medium)    & 223 & 56 & 57 & 39 & 71 & 51\% & \cellcolor[HTML]{F5F3FC}0.426 & \cellcolor[HTML]{FFF6E9}0.496 & \cellcolor[HTML]{FFF0DC}\color[HTML]{1B2735}$+7.0$ \\
\rowcolor[HTML]{F5F6F7}
6             & 0   & 0  & 0  & 0  & 0  & ---  & ---   & ---   & ---    \\
7 (hard)      & 100 & 17 & 25 & 15 & 43 & 42\% & \cellcolor[HTML]{F5F3FC}0.320 & \cellcolor[HTML]{FFF6E9}0.405 & \cellcolor[HTML]{FFF0DC}\color[HTML]{1B2735}$+8.5$ \\
8 (very hard) & 123 & 22 & 22 & 29 & 50 & 36\% & \cellcolor[HTML]{F5F3FC}0.415 & \cellcolor[HTML]{FFF6E9}0.500 & \cellcolor[HTML]{FFF0DC}\color[HTML]{1B2735}$+8.5$ \\
\rowcolor[HTML]{F5F6F7}
9             & 0   & 0  & 0  & 0  & 0  & ---  & ---   & ---   & ---    \\
\rowcolor[HTML]{F5F6F7}
10            & 0   & 0  & 0  & 0  & 0  & ---  & ---   & ---   & ---    \\
11 (hardest)  & 103 & 19 & 12 & 38 & 34 & 30\% & \cellcolor[HTML]{F5F3FC}0.553 & \cellcolor[HTML]{FFF6E9}0.613 & \cellcolor[HTML]{FFF0DC}\color[HTML]{1B2735}$+6.0$ \\
\midrule
\rowcolor[HTML]{EAF1FC}
\textbf{Total} & \textbf{1000} & \textbf{299} & \textbf{216} & \textbf{208} & \textbf{277} & & & & \\
\bottomrule
\end{tabularx}
\end{table}

Easy strata (1--3 operations) have small gaps ($-1.1$ to $+2.8$pp), higher accuracy, and less selective filtering than the remaining strata. Medium-to-hard strata (5--8) have the largest gaps ($+7.0$ to $+8.5$pp), with a smaller fraction retained by the filter. The operation-count/gap rank association is positive but non-significant (Spearman $\rho = 0.75$, $p = 0.052$, $n = 7$ strata).

At 11 operations the gap is smaller ($+6.0$pp) than at 5--8. This stratum's $\Fpop$ ($0.553$) exceeds the values at 3--8 operations, but remains below those at 1 and 2 operations ($0.922$ and $0.750$). Its null rate is close to the pooled rate ($7.8\%$ versus $7.2\%$). Sampling variation in the small stratum ($n=103$) and differences in graph structure are both possible explanations. This stratum contains the deepest graphs, where sequential dependencies can be easier than branching ones. The current data do not distinguish these explanations.

Operation count therefore does not consistently order the strata by difficulty. The 42.8\% between-stratum share reported in Section~\ref{sec:empirical} quantifies \emph{variation explained by the operation count}.

\subsection{Exact Covariance Decomposition}
\label{sec:decomp_details}

Table~\ref{tab:decomp} reports the decomposition from Section~\ref{sec:empirical}, computed by the law of total covariance \eqref{eq:total_cov} with interventionable-operation count as the stratifier. The identity holds to floating-point precision without an independence assumption. It partitions the observed covariance; the IPW analysis below instead reweights selected instances under a propensity model.

From \eqref{eq:bridge}, $\Delta = \Cov(A,C)/P(C{=}1)$ and $1/P(C{=}1)$ is constant within a model's run. The between- and within-stratum terms of \eqref{eq:total_cov} partition $\Cov(A,C)$. Multiplying both terms by the same constant partitions $\Delta$, so the covariance and gap have the same percentage split. Table~\ref{tab:decomp} reports shares \emph{of the gap}.

Table~\ref{tab:twobytwo} gives each model's four $(A,C)$ cells under the null-as-non-compliant convention. Tables~\ref{tab:depth_strat} and~\ref{tab:strata_flash} give the per-stratum cells. These printed counts reproduce every point-estimate share in Table~\ref{tab:decomp}.

The joint counts give $\Fsel=a/(a+b)$ and $\Fpop=(a+c)/n$. The odds ratio is $ad/(bc)$, with a $0.5$ continuity correction if a cell is zero. The pooled Cochran--Mantel--Haenszel analysis uses the four sub-ceiling deployments.

\begin{table}[htbp]
\caption{\textbf{Joint compliance and correctness counts.} Ceiling controls are excluded from the pooled analysis.}
\label{tab:twobytwo}
\centering
\fontsize{9}{10.8}\selectfont
\setlength{\tabcolsep}{4pt}
\renewcommand{\arraystretch}{1.18}
\arrayrulecolor{black!35}
\setlength{\heavyrulewidth}{0.65pt}
\setlength{\lightrulewidth}{0.35pt}
\begin{tabularx}{\linewidth}{@{}>{\raggedright\arraybackslash}Xrcccccc@{}}
\toprule
\rowcolor[HTML]{F1F6FE}
\textbf{\boldmath Model} & \textbf{\boldmath $n$} & \textbf{\boldmath $a$} & \textbf{\boldmath $b$} & \textbf{\boldmath $c$} & \textbf{\boldmath $d$} & \textbf{\boldmath $\Delta$ (pp)} & \textbf{\boldmath OR} \\

\midrule
\rowcolor[HTML]{EAF1FC}
GLM-5.3 & 1000 & 299 & 216 & 208 & 277 & \cellcolor[HTML]{FFF0DC}\color[HTML]{1B2735}$+7.4$ & 1.84 \\
deepseek-v4-flash & 300 & 85 & 60 & 73 & 82 & \cellcolor[HTML]{FFF0DC}\color[HTML]{1B2735}$+6.0$ & 1.59 \\
deepseek-v4-pro & 200 & 77 & 56 & 34 & 33 & \cellcolor[HTML]{FFF0DC}\color[HTML]{1B2735}$+2.4$ & 1.33 \\
glm-5.2 & 100 & 36 & 43 & 8 & 13 & \cellcolor[HTML]{FFF0DC}\color[HTML]{1B2735}$+1.6$ & 1.36 \\
\midrule
\rowcolor[HTML]{F5F6F7}
\multicolumn{8}{l}{\emph{Boundary controls ($P(C{=}1)=1$; $c=d=0$ by construction):}} \\
qwen3.8-max & 20 & 12 & 8 & 0 & 0 & \cellcolor[HTML]{F5F6F7}$0.0$ & --- \\
Gemini 3.5 Flash Lite & 17 & 9 & 8 & 0 & 0 & \cellcolor[HTML]{F5F6F7}$0.0$ & --- \\
\bottomrule
\end{tabularx}
\end{table}

Figure~\ref{fig:selection_detail} shows how the GLM-5.3 joint counts determine the reporting-population contrast.

\begin{figure}[!htbp]
\centering
\includegraphics[width=\textwidth]{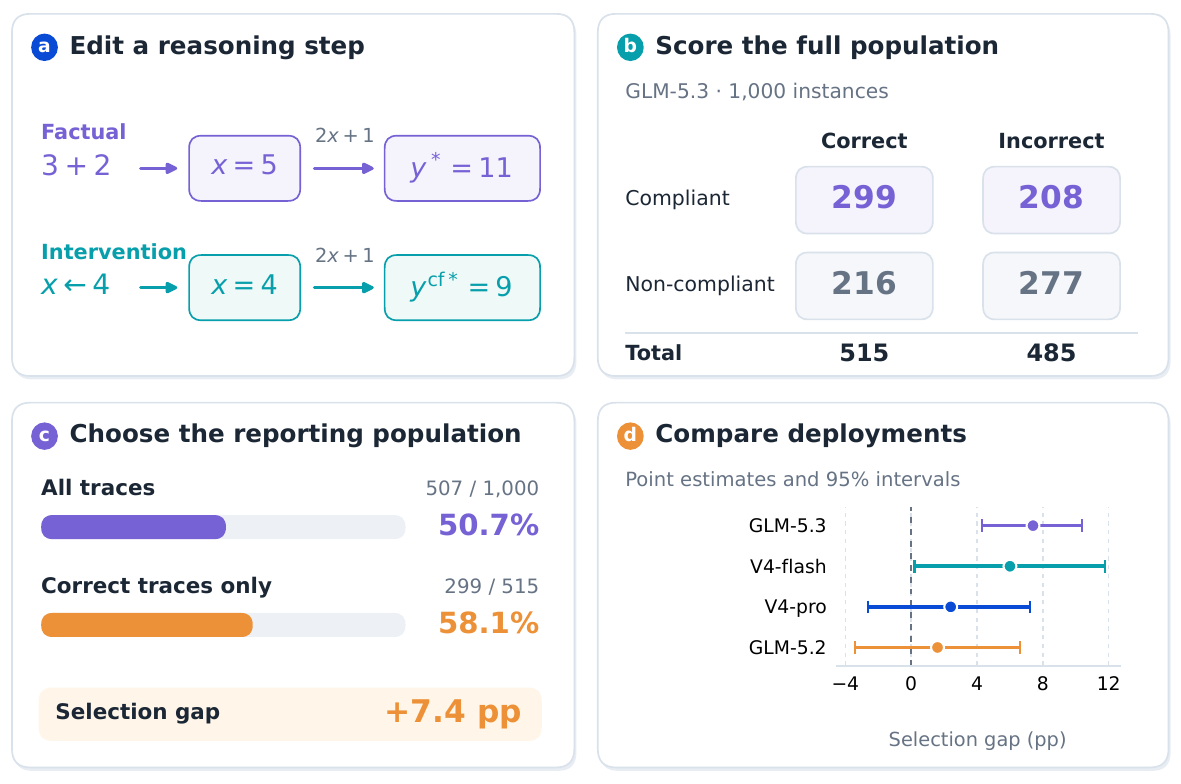}
\caption{\textbf{Worked selection calculation.} An illustrative intervention and the joint counts connecting oracle compliance to all-trace and correctness-selected estimates.}
\label{fig:selection_detail}
\end{figure}

The descriptive 95\% percentile intervals in Table~\ref{tab:decomp} use an instance-level bootstrap with 10,000 resamples. Row-level data and exploratory recomputation outputs are not included in the public release. The share divides by $\Cov(A,C)$, which can approach zero during resampling. This produces asymmetric intervals; the flash interval crosses zero.

Operation count gives a 42.8\% between-stratum share for GLM-5.3, while oracle-displacement magnitude gives 26.6\% on the same records. These covariate-dependent shares do not identify a causal contribution of latent difficulty.

\begin{table}[htbp]
\caption{\textbf{Covariance decomposition by deployment.} Brackets report bootstrap intervals; the within-stratum component remains unattributed.}
\label{tab:decomp}
\centering
\fontsize{9}{10.8}\selectfont
\setlength{\tabcolsep}{4pt}
\renewcommand{\arraystretch}{1.18}
\arrayrulecolor{black!35}
\setlength{\heavyrulewidth}{0.65pt}
\setlength{\lightrulewidth}{0.35pt}
\begin{tabularx}{\linewidth}{@{}>{\raggedright\arraybackslash}Xcccc@{}}
\toprule
\rowcolor[HTML]{F1F6FE}
\textbf{\boldmath Model} & \textbf{\boldmath $n$} & \textbf{\boldmath $\Delta$ (pp)} & \textbf{\boldmath Between-stratum} & \textbf{\boldmath Within-stratum} \\

\midrule
\rowcolor[HTML]{EAF1FC}
GLM-5.3 & 1000 & \cellcolor[HTML]{FFF0DC}\color[HTML]{1B2735}7.36 & 42.8\% \;[26.8, 69.1] & 57.2\% \\
deepseek-v4-flash & 300 & \cellcolor[HTML]{FFF0DC}\color[HTML]{1B2735}5.95 & 21.9\% \;[$-39.0$, 112.2] & 78.1\% \\
\bottomrule
\end{tabularx}
\end{table}

The flash strata in Table~\ref{tab:strata_flash} comprise $n=300$ instances and yield the 21.9\% between-stratum share in Table~\ref{tab:decomp}. The corresponding GLM-5.3 strata are given in Table~\ref{tab:depth_strat}.

\begin{table}[htbp]
\caption{\textbf{DeepSeek-v4-flash strata.} Joint counts by interventionable-operation count for covariance decomposition.}
\label{tab:strata_flash}
\centering
\fontsize{9}{10.8}\selectfont
\setlength{\tabcolsep}{4pt}
\renewcommand{\arraystretch}{1.18}
\arrayrulecolor{black!35}
\setlength{\heavyrulewidth}{0.65pt}
\setlength{\lightrulewidth}{0.35pt}
\begin{tabularx}{\linewidth}{@{}>{\raggedright\arraybackslash}Xrcccc@{}}
\toprule
\rowcolor[HTML]{F1F6FE}
\textbf{\boldmath Steps} & \textbf{\boldmath $n$} & \textbf{\boldmath $a$} & \textbf{\boldmath $b$} & \textbf{\boldmath $c$} & \textbf{\boldmath $d$} \\

\midrule
1 (simple) & 36 & 21 & 1 & 12 & 2 \\
2 (easy) & 42 & 9 & 9 & 9 & 15 \\
3 (moderate) & 63 & 12 & 21 & 9 & 21 \\
5 (medium) & 59 & 15 & 12 & 17 & 15 \\
7 (hard) & 33 & 5 & 8 & 8 & 12 \\
8 (very hard) & 37 & 13 & 5 & 11 & 8 \\
11 (hardest) & 30 & 10 & 4 & 7 & 9 \\
\midrule
\rowcolor[HTML]{EAF1FC}
\textbf{Total} & \textbf{300} & \textbf{85} & \textbf{60} & \textbf{73} & \textbf{82} \\
\bottomrule
\end{tabularx}
\end{table}

\subsection{Inverse Propensity Weighting (IPW) Model Details}
\label{sec:ipw_details}

We use an inverse-probability-weighted estimator~\citep{horvitz1952generalization}. We fit a logistic propensity model, $\hat{P}(C{=}1 \mid \mathrm{steps}) = \mathrm{logistic}(\hat{\beta}_0 + \hat{\beta}_1 \cdot \mathrm{steps})$, where $\mathrm{steps}$ is the number of interventionable operations ($\hat{\beta}_0 = 0.954$, $\hat{\beta}_1 = -0.181$). Importance reweighting is a standard correction for sample-selection bias in supervised learning~\citep{zadrozny2004learning,cortes2008sample}.

We apply weights $w_i = 1 / \hat{P}(C{=}1 \mid D_i)$ to all correctness-selected instances. The numerator sums weights for compliant selected instances; the denominator sums weights for all selected instances. Under $A \perp C \mid D,I$, this normalized weighted mean is a Hájek estimator of $\Fpop$. The definition of $\Fsel$ is unchanged. We apply no weight trimming or capping. The effective sample size is $n_{\mathrm{eff}} = (\sum w_i)^2 / \sum w_i^2 = 472$, with mean weight 1.95, median 1.66, and maximum 3.83.

The logistic propensity gives $\hat{F}_{\mathrm{IPW}} = 0.564$, leaving a $+5.7$pp residual relative to the direct $\Fpop=0.507$, a 23\% reduction from the raw selected gap. A nonparametric per-stratum propensity gives $\hat{F}_{\mathrm{IPW}} = 0.555$, a 34\% reduction. Both adjustments depend on the propensity specification, and neither identifies a causal decomposition. The residual remains unattributed.

\section{Reversal Sensitivity Analysis}
\label{sec:reversal_sensitivity}

Table~\ref{tab:reversal_sensitivity} varies $\pi_{10}=P(C=1\mid R=1,S=0)$ while holding $\pi_{00}=0.05$, $\pi_{01}=0.80$, and $\pi_{11}=1.00$. The reference model has $(p_{r,B},p_{s,B})=(0.70,0.01)$. The full candidate grid comprises 199 equally spaced $p_{r,A}$ values from 0.01 to 0.99 and 201 equally spaced $p_{s,A}$ values from 0 to 1: 39,999 candidates in total. We count a reversal only when $p_{r,A}<p_{r,B}$ and $\Fsel(\mathcal A)>\Fsel(\mathcal B)$, using strict inequalities.

Figure~\ref{fig:reversal} shows $\pi_{10}=0.001$, where faithful reasoning without shortcut support almost never succeeds. Shortcut prevalence varies across candidates; $S$ is not fixed to one. The plot spans $(p_{r,A},p_{s,A})\in[0,1]^2$, and the reported 40.6\% is $16{,}224/39{,}999$ on the discrete full grid specified above. Candidates with $p_{r,A}\geq0.70$ do not meet the lower-population-faithfulness condition.

\begin{figure}[htbp]
\centering
\includegraphics[width=0.75\textwidth]{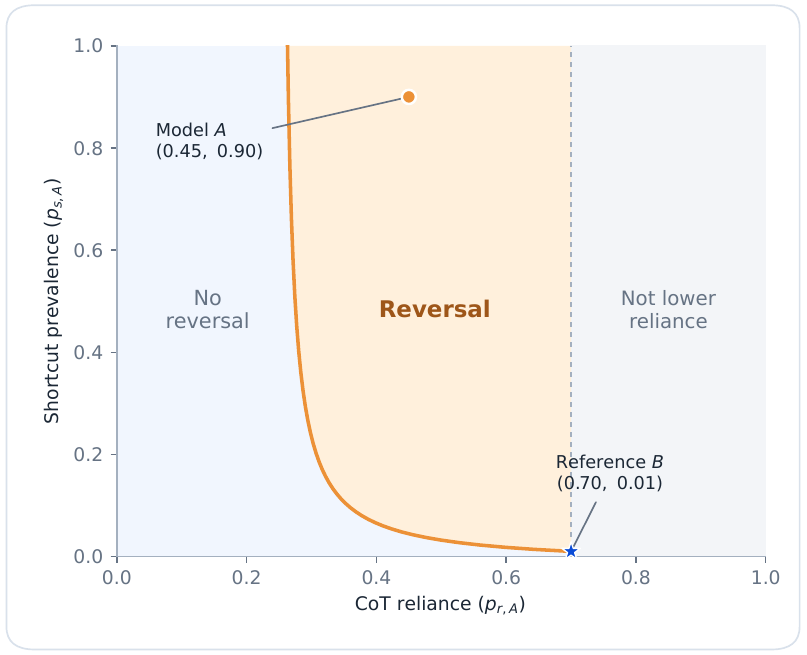}
\caption{\textbf{Ranking reversals at $\pi_{10}=0.001$.} The orange region contains candidates with lower population faithfulness but higher selected-trace scores than reference $\mathcal B$.}
\label{fig:reversal}
\end{figure}

Independent exact-rational evaluation reproduces the counts below. Requiring both inequality margins to exceed $10^{-8}$ gives the same counts, excluding numerical ties as their source.

\begin{table}[htbp]
\caption{\textbf{Reversal sensitivity.} Counts and percentages on the same 39,999-candidate grid, with the reference and other correctness parameters held fixed.}
\label{tab:reversal_sensitivity}
\centering
\fontsize{9}{10.8}\selectfont
\setlength{\tabcolsep}{4pt}
\renewcommand{\arraystretch}{1.18}
\arrayrulecolor{black!35}
\setlength{\heavyrulewidth}{0.65pt}
\setlength{\lightrulewidth}{0.35pt}
\begin{tabularx}{\linewidth}{@{}>{\raggedright\arraybackslash}Xcc@{}}
\toprule
\rowcolor[HTML]{F1F6FE}
\textbf{\boldmath $\pi_{10}$} & \textbf{\boldmath Reversal fraction} & \textbf{\boldmath Reversal count} \\

\midrule
0.001 & \cellcolor[HTML]{F5F6F7}40.561\% & 16,224 \\
0.05 & \cellcolor[HTML]{F5F6F7}3.568\% & 1,427 \\
0.1 & \cellcolor[HTML]{F5F6F7}0.0075\% & 3 \\
0.2 & \cellcolor[HTML]{F5F6F7}0.0150\% & 6 \\
\bottomrule
\end{tabularx}
\end{table}

At $\pi_{10}=0.001$ and $0.05$, every reversal has $p_{s,A}>0.01$, consistent with the positive complementarity determinant $\pi_{11}\pi_{00}-\pi_{10}\pi_{01}$. At $\pi_{10}=0.1$ and $0.2$, that determinant is negative, ruling out reversals with $p_{s,A}>p_{s,B}$. The 3 and 6 remaining reversals instead have lower shortcut prevalence: respectively 2 and 4 occur at $p_{s,A}=0$, and 1 and 2 at $p_{s,A}=0.005$. The table describes these four settings; it does not establish that reversal is impossible for every $\pi_{10}\geq0.2$.

\section{Leave-One-Out Sensitivity Analysis}
\label{sec:leave_one_out}

Table~\ref{tab:leave_one_out} reports the CMH test after removing each model in turn. Figure~\ref{fig:forest_plot} shows the deployment-specific and pooled odds ratios. All leave-one-out tests remain significant ($p < 0.03$), so the pooled association remains detectable after removing any single model. GLM-5.3 is the largest single contributor.

Across the four sub-ceiling models, the pooled test gives $z = 5.19$, $p < 10^{-6}$, OR $= 1.70$, and $I^2 = 0\%$. Gaps also persist within difficulty strata, though their rank association is only suggestive ($\rho = 0.75$, $p = 0.052$; Appendix~\ref{sec:depth_analysis}). The per-model ORs are 1.33, 1.59, 1.36, and 1.84 across the combined sample ($n = 1600$). The null-compliance rate is 5.44\% ($87/1600$); nulls are scored non-compliant under the headline convention, with a sensitivity check in Appendix~\ref{sec:compliance_discussion}. Elicitation found 10\% cue use in GLM-5.3 (Appendix~\ref{sec:shortcut_elicitation}).

\begin{table}[htbp]
\caption{\textbf{Leave-one-out analysis.} Pooled association after omitting each deployment.}
\label{tab:leave_one_out}
\centering
\fontsize{9}{10.8}\selectfont
\setlength{\tabcolsep}{4pt}
\renewcommand{\arraystretch}{1.18}
\arrayrulecolor{black!35}
\setlength{\heavyrulewidth}{0.65pt}
\setlength{\lightrulewidth}{0.35pt}
\begin{tabularx}{\linewidth}{@{}>{\raggedright\arraybackslash}Xccc@{}}
\toprule
\rowcolor[HTML]{F1F6FE}
\textbf{\boldmath Removed} & \textbf{\boldmath Remaining $n$} & \textbf{\boldmath $z$} & \textbf{\boldmath $p$ (two-sided)} \\

\midrule
\rowcolor[HTML]{EAF1FC}
GLM-5.3 & 600 & 2.24 & 0.025 \\
deepseek-v4-flash & 1300 & 4.80 & $<$0.0001 \\
deepseek-v4-pro & 1400 & 5.17 & $<$0.0001 \\
GLM-5.2 & 1500 & 5.17 & $<$0.0001 \\
\bottomrule
\end{tabularx}
\end{table}

\begin{figure}[htbp]
\centering
\includegraphics[width=\textwidth]{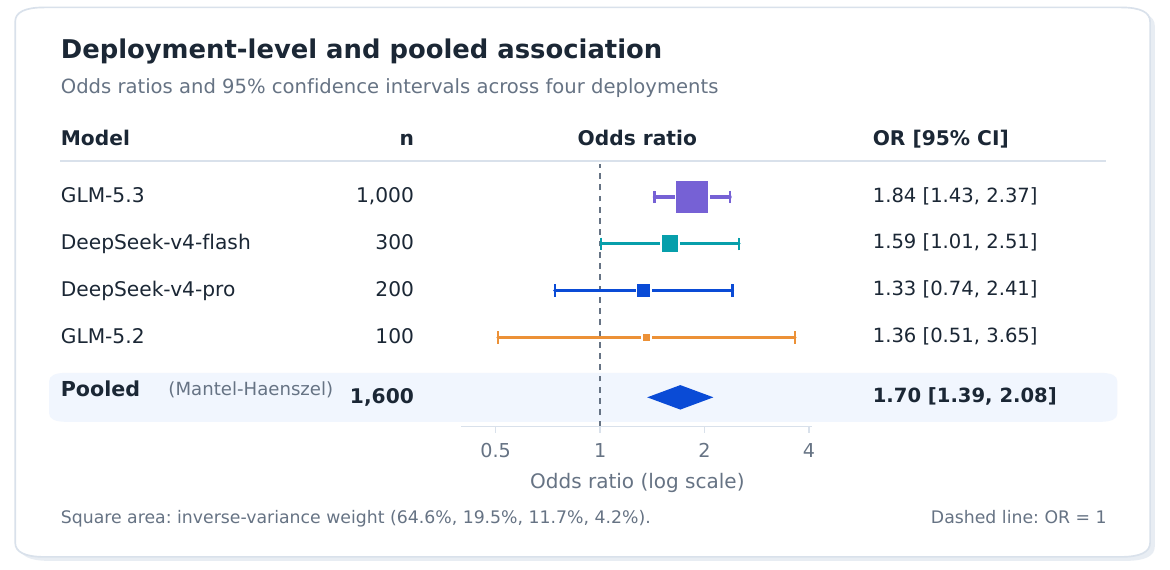}
\caption{\textbf{Deployment-level and pooled odds ratios.} Square areas encode inverse-variance weights; horizontal lines show individual confidence intervals and the diamond the pooled interval.}
\label{fig:forest_plot}
\end{figure}

\section{Shortcut Elicitation Details}
\label{sec:shortcut_elicitation}

Following prior work on biased feature cues~\citep{al2023language}, we applied adversarial framing cues (authority, social proof, and shortcut suggestion) to 50 arithmetic DAG instances per model. GLM-5.3 exhibited 10\% shortcut use overall (5/50, 95\% Wilson CI [4.3\%, 21.4\%]), compared to 0\% under neutral framing; deepseek-v4-flash showed 0\%. The observed GLM-5.3 gap (+3.0pp) is directionally consistent with the theory but not significant at $n = 50$. All five shortcuts occurred on correct-cue instances and passed the correctness filter; these observations do not quantify a collider contribution.

\paragraph{Targeted elicitation on a second deployment.} We vary the framing of the same arithmetic DAG task. V1 asks the model to verify a stated value; V2 presents it as descriptive social proof; V4 explicitly instructs the model to rely on a stated ``definitive'' value. Under V4, deepseek-v4-pro followed the cue on 12 of 12 informative (incorrect-value) instances and answered all 12 incorrectly. The correctness filter removes all of these rows. Paired controls for three instances span easy and hard tiers: the same task is answered correctly without a cue and incorrectly with one.

Under V2, the same model followed the cue on 0 of 4 informative instances, computing correctly and explicitly rejecting it. Flash followed the cue on 0 of 5 informative instances and glm-5.2 on 0 of 2; neither model followed it under any tested framing. One V1 call returned no completed answer and was not scored. The observed following rate was approximately $0$ under V2, unmeasurable under V1, and approximately $1$ under V4, with no framing producing partial following in this probe. Following depended on both framing and model; these observations do not support interpolating a graded response.

In a representative case, the model derived the correct value from the DAG but substituted the cue in its final answer. The error occurred at the \emph{output stage}. The correctness filter excludes this case despite the trace's correct computation.

This exploratory probe identifies framing-dependent cue following in small, post-hoc samples. It was not pre-registered, uses a separate serving endpoint from Table~\ref{tab:empirical}, and covers only part of the hard-task tier. We therefore report it separately from the headline runs. Estimating a population rate or testing the accuracy-axis peak requires a larger design with partial-following conditions.

\paragraph{Cross-family extension.} We apply the V4/V2/V1 framings to additional model families. Each row reports a separate deployment; we do not pool across endpoints.

\begin{center}
\fontsize{9}{10.8}\selectfont
\setlength{\tabcolsep}{4pt}
\renewcommand{\arraystretch}{1.18}
\arrayrulecolor{black!35}
\setlength{\heavyrulewidth}{0.65pt}
\setlength{\lightrulewidth}{0.35pt}
\begin{tabularx}{\linewidth}{@{}>{\raggedright\arraybackslash}Xcc@{}}
\toprule
\rowcolor[HTML]{F1F6FE}
\textbf{\boldmath Model family} & \textbf{\boldmath Measured} & \textbf{\boldmath Followed} \\

\midrule
gemini-2.5-flash & 6 & \cellcolor[HTML]{F5F6F7}6 \\
deepseek-v4-pro & 12 & \cellcolor[HTML]{F5F6F7}12 \\
kimi-k3 & 6 & \cellcolor[HTML]{F5F6F7}2 \\
glm-5.2 & 5 & \cellcolor[HTML]{F5F6F7}0 \\
deepseek-v4-flash & 2 & \cellcolor[HTML]{F5F6F7}0 \\
\bottomrule
\end{tabularx}
\end{center}
\noindent\textit{Counts are measured rows; calls without a completed answer are excluded from the rates.}

Across the single-endpoint and cross-family probes combined, V2 yielded 0 of 14 informative follows and V1 yielded 0 of 7. V4 was the only framing to induce following. The rates range from 0\% to 100\% across families, consistent with a framing $\times$ family interaction.

\paragraph{V4-neutral control.} The V4-neutral condition uses the same wording instructing deference but sets the stated value to the \emph{correct} answer. In 25 measured neutral rows, no deployment cited an instruction to skip computation. These rows include 13 discriminating pairs, in which the deployment followed under V4-wrong and was also tested under V4-neutral. Gemini supplies behavioural evidence because it returns no plaintext reasoning. In the other paired cases, compliance language appears when the stated value differs from the computed value and disappears when the stated value is correct. This contrast is observed; interpreting it as a consistency check remains a hypothesis that the sample does not independently confirm.

All cross-family comparisons are also cross-endpoint. The probe covers only part of the hard-task tier. Gemini's endpoint quota limited its sample size, and one model family used a non-default temperature because its endpoint rejected temperature $=0$. These conditions limit the comparisons to exploratory post-hoc observations.

\paragraph{Selection-gap estimates under adversarial framing.} The 100 evaluation records used to estimate shortcut use also have factual and counterfactual answers, oracle values, correctness labels, and compliance scores. They use the same arithmetic DAGs as the headline evaluation, with operation counts $\in \{1,2,3,5,7,8,11\}$. We apply the same estimator to this second \emph{cue regime}; it adds no new task family.

Table~\ref{tab:turpin_estimator} uses the earlier adversarial run, with $n=50$ per model on the same endpoints as Table~\ref{tab:empirical}; it does not summarize the later targeted probes. Its tests are Yates-corrected $\chi^2$ tests with one degree of freedom, and neither gap is significant. For flash, the incorrect-cue joint cells $(a,b,c,d)=(0,5,3,6)$ determine the overall gap's negative sign.

\begin{table}[htbp]
\caption{\textbf{Selection gaps under adversarial framing.} Estimates from the earlier run, separate from the subsequent exploratory probes.}
\label{tab:turpin_estimator}
\centering
\fontsize{9}{10.8}\selectfont
\setlength{\tabcolsep}{4pt}
\renewcommand{\arraystretch}{1.18}
\arrayrulecolor{black!35}
\setlength{\heavyrulewidth}{0.65pt}
\setlength{\lightrulewidth}{0.35pt}
\begin{tabularx}{\linewidth}{@{}>{\raggedright\arraybackslash}Xcccccc@{}}
\toprule
\rowcolor[HTML]{F1F6FE}
\textbf{\boldmath Model} & \textbf{\boldmath $n$} & \textbf{\boldmath Acc.} & \textbf{\boldmath $\Fpop$} & \textbf{\boldmath $\Fsel$} & \textbf{\boldmath $\Delta$ (pp)} & \textbf{\boldmath $\chi^2$ ($p$)} \\

\midrule
\rowcolor[HTML]{EAF1FC}
GLM-5.3 & 50 & 78\% & \cellcolor[HTML]{F5F3FC}0.560 & \cellcolor[HTML]{FFF6E9}0.590 & \cellcolor[HTML]{FFF0DC}\color[HTML]{1B2735}$+3.0$ & 0.206 (0.650) \\
deepseek-v4-flash & 50 & 44\% & \cellcolor[HTML]{F5F3FC}0.220 & \cellcolor[HTML]{FFF6E9}0.182 & \cellcolor[HTML]{EDF8F9}\color[HTML]{076872}$-3.8$ & 0.055 (0.815) \\
\bottomrule
\multicolumn{7}{l}{\footnotesize Breakdown by cue correctness: flash correct-cue $+1.31$pp ($n{=}36$), incorrect-cue $-21.43$pp ($n{=}14$);} \\
\multicolumn{7}{l}{\footnotesize GLM-5.3 correct-cue $+4.50$pp ($n{=}36$), incorrect-cue $+1.43$pp ($n{=}14$).} \\
\end{tabularx}
\end{table}

Both gaps are non-significant at $n = 50$, with $39$ and $22$ correct rows respectively. The pilot power range is $0.06$--$0.25$ at $n = 100$ (Appendix~\ref{sec:power}), so these smaller runs have limited power. The estimates describe selection in the same data used to measure shortcut use.

The negative flash estimate does not identify a mechanism. The collider pathway requires an \emph{active} shortcut branch, and flash shows $0/50$ shortcuts here (all five elicited shortcuts are GLM-5.3's, and all fall on correct-cue instances that pass the filter). The negative overall figure is a $4$-versus-$4.84$ cell difference concentrated in $14$ incorrect-cue rows whose selected-trace numerator is zero; this cell difference provides no evidence for a specific pathway.

\section{Empirical Gap vs.\ Accuracy}
\label{sec:goldilocks_fig}

Figure~\ref{fig:goldilocks} shows the deployment estimates with complete confidence intervals.

\begin{figure}[!htbp]
\centering
\includegraphics[width=\textwidth]{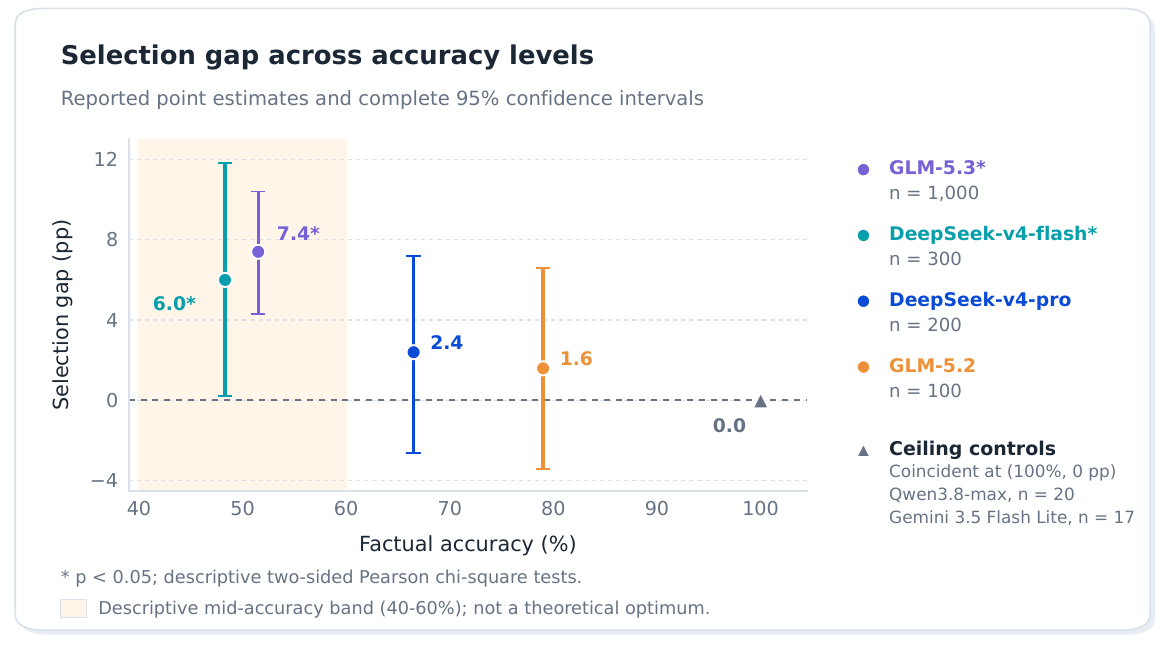}
\caption{\textbf{Selection gap versus accuracy.} All confidence intervals are shown; asterisks mark descriptive $p<0.05$.}
\label{fig:goldilocks_original}
\label{fig:goldilocks}
\end{figure}

The shaded mid-accuracy band is descriptive. The theory instead places the peak at $P^* = \sqrt{p_s}$ (Corollary~\ref{cor:accuracy_hump}). At the implementation-check value $p_s = 0.7$, the peak is $0.84$, above every tested sub-ceiling deployment. The two ceiling controls also do not locate the peak: $P(C{=}1) = 1$ makes the filter vacuous and forces $\Delta = 0$ by construction.

\section{Ceiling Validation}
\label{sec:ceiling}

The two ceiling controls use task difficulties at which every valid response is factually correct. Gemini 3.5 Flash Lite (Google) returned 17 valid responses from 20 requests (three were unparseable), with $\Fsel=\Fpop=0.529$; qwen3.8-max has $n=20$, 100\% accuracy, and $\Delta=0$. Their lengths are unmatched, so the numerical scores are not directly comparable. In both constructed evaluation sets, $P(C{=}1)=1$ makes the filter vacuous and forces $\Delta=0$ (Proposition~\ref{prop:collider_bias}). These runs check the implementation at the boundary; they do not independently establish a selection effect.

On Boolean circuits, the same all-trace protocol with GLM-5.3 ($n=200$) gives 99.5\% accuracy, $\Fpop=\Fsel=0.995$, and $\Delta=0.0$pp, compared with $+7.4$pp on arithmetic DAGs at intermediate accuracy. This second-family control agrees with the predicted limit $\Delta\to0$ as $P(C{=}1)\to1$ (Section~\ref{sec:simulation}); the circuits were too easy to measure the model below ceiling.

At depth~30 (90 gates), we verified non-thinking mode on every row (\texttt{reasoning\_tokens}$=0$), with no parse failures. Deepseek-v4-flash answered 10/10 correctly (one-sided 95\% lower bound $0.741$); glm-5.2 answered all 10 calls correctly (bound $0.688$ at the conservative $n=8$ distinct instances). Both bounds exclude the 40--60\% band. These two deployments are separate from the GLM-5.3 control.

A further deepseek-v4-flash test increased the serial chain from 90 to 120 to 300 gates (depths~30, 40, and~100); accuracy remained 10/10. The relay's 180\,s limit prevented measurement at 600 gates. Thus the reachable circuit depths provide no non-ceiling measurement. Arithmetic DAGs instead reach 48--52\% accuracy at depths 2--4. This contrast documents reachable difficulty for the tested deployments, not a difference in mechanism or an extension of the selection-bias evidence.

\section{Compliance on Incorrect Traces and Null Handling}
\label{sec:compliance_discussion}

Compliance on incorrect traces (38--51\% across models) is semantically ambiguous: it may reflect genuine partial faithfulness, coincidental answer changes, or error-correlated sensitivity. All-trace evaluation includes these instances, so $\Fpop$ is informative if partial faithfulness dominates and inflated if coincidental changes do; the estimator is not uniformly conservative in either direction.

\emph{Null handling.} Our primary protocol scores unparseable or missing responses as non-compliant ($A{=}0$). Nulls are unevenly distributed across instances, so assigning $A{=}0$ can move the gap in either direction. We test this convention by treating null responses as missing and dropping them from both $\Fpop$ and $\Fsel$, using the same instance sets as the headline analysis:

\begin{center}
\fontsize{9}{10.8}\selectfont
\setlength{\tabcolsep}{4pt}
\renewcommand{\arraystretch}{1.18}
\arrayrulecolor{black!35}
\setlength{\heavyrulewidth}{0.65pt}
\setlength{\lightrulewidth}{0.35pt}
\begin{tabularx}{\linewidth}{@{}>{\raggedright\arraybackslash}Xccccc@{}}
\toprule
\rowcolor[HTML]{F1F6FE}
\textbf{\boldmath Model} & \textbf{\boldmath $n$ (noncomp)} & \textbf{\boldmath $n$ (missing)} & \textbf{\boldmath $\Delta$ (noncomp)} & \textbf{\boldmath $\Delta$ (missing)} & \textbf{\boldmath shift} \\

\midrule
\rowcolor[HTML]{EAF1FC}
GLM-5.3 & 1000 & 928 & \cellcolor[HTML]{FFF0DC}\color[HTML]{1B2735}+7.36pp & \cellcolor[HTML]{FFF0DC}\color[HTML]{1B2735}+7.14pp & \cellcolor[HTML]{FFF0DC}\color[HTML]{1B2735}$+0.22$ \\
flash & 300 & 291 & \cellcolor[HTML]{FFF0DC}\color[HTML]{1B2735}+5.95pp & \cellcolor[HTML]{FFF0DC}\color[HTML]{1B2735}+5.15pp & \cellcolor[HTML]{FFF0DC}\color[HTML]{1B2735}$+0.81$ \\
v4-pro & 200 & 198 & \cellcolor[HTML]{FFF0DC}\color[HTML]{1B2735}$+2.39$pp & \cellcolor[HTML]{FFF0DC}\color[HTML]{1B2735}$+2.27$pp & \cellcolor[HTML]{FFF0DC}\color[HTML]{1B2735}$+0.12$ \\
glm-5.2 & 100 & 96 & \cellcolor[HTML]{FFF0DC}\color[HTML]{1B2735}$+1.57$pp & \cellcolor[HTML]{FFF0DC}\color[HTML]{1B2735}$+0.92$pp & \cellcolor[HTML]{FFF0DC}\color[HTML]{1B2735}$+0.65$ \\
\bottomrule
\end{tabularx}
\end{center}

Scoring nulls as non-compliant increases the gap by $+0.12$ to $+0.81$pp relative to dropping them. Nulls fall disproportionately on incorrect instances (for GLM-5.3, 41 of 72 nulls are on incorrect rows against a 48.5\% base rate), so assigning $A{=}0$ depresses $\Fpop$ more than $\Fsel$. All four gap estimates remain positive under both conventions.

\section{Full Discussion}
\label{sec:discussion_full}

\emph{Pathway attribution.} The detector flagged no neutral-framing instances. A flag requires the answer to equal the cue and either the trace to name it (``reference'', ``previously'', ``provided result'') or the answer to diverge from the oracle. Unacknowledged use of a correct cue escapes detection; 72\% of cues were correct. The zero rate therefore cannot exclude collider bias ($R\to C\leftarrow S$). Operation-count strata capture 21.9\% (flash) and 42.8\% (GLM-5.3) of the gap, consistent with the difficulty paths $D\to C$ and $D\to A$, but most remains unattributed.

Because interventions leave the cue unchanged, cue correctness provides a descriptive check. In the $n=300$ arithmetic flash run, complete-case compliance is $53.8\%$ (112/208) with correct cues and $55.4\%$ (46/83) with incorrect cues: $-1.6$pp ($z=-0.24$, $p=0.81$). This excludes the 9 nulls retained in the headline estimate. The non-significant contrast does not resolve cue reliance, since cue use in both groups can also produce similar rates.

Adversarial elicitation flags 10\% of GLM-5.3 responses (5/50, Section~\ref{sec:shortcut_elicitation}); all five use correct cues and pass the filter. The targeted incorrect-cue probe uses separate deployments: under explicit-reliance framing, deepseek-v4-pro (12/12), gemini-2.5-flash (6/6), and kimi-k3 (2/6) follow the stated value on informative instances, producing wrong answers excluded by the filter. Neither glm-5.2 nor flash follows in any tested framing. These unpooled probes establish cue following and exclusion, not its contribution to the headline gap. Both pathways can coexist under Proposition~\ref{prop:collider_bias}.

\emph{Difficulty adjustment and residual.} Logistic IPW removes 23\% of the gap, leaving a significantly positive residual. This single-covariate adjustment depends on the propensity specification and does not identify the residual's mechanism.

\section{Full Limitations}
\label{sec:limitations_full}

\emph{Scope.} Two of the four sub-ceiling models reach individual significance ($p < 0.05$). The other two have positive point estimates with intervals crossing zero at $n = 100$--$200$. The two 100\%-accuracy controls are algebraic identities and are excluded from the pooled test; Table~\ref{tab:empirical} reports their $p$-values as N/A.

All headline real-model runs use arithmetic DAGs. Boolean circuits were evaluated only at ceiling and therefore check the boundary condition. We also attempted graph reachability with GLM-5.3 ($n = 20$), but excluded it from the evidence: accuracy was 95\%, and only 3 of 20 counterfactual responses parsed, with the rest empty. This run supports no inference about the mechanism. No family other than arithmetic has been measured at non-ceiling accuracy. Generalization to natural-language CoT, a central setting in the reviewed literature, remains untested.

\emph{Mechanism.} The collider pathway ($R \to C \leftarrow S$) requires a shortcut branch, which was not detected under neutral framing. A targeted probe reaches it on multiple deployments under an explicit-reliance framing (Section~\ref{sec:shortcut_elicitation}), using different endpoints without pooling. For the headline gaps, the measured operation-count association is consistent with difficulty confounding; it does not identify a causal contribution.

\emph{Residual.} IPW removes 23--34\% of the gap across propensity specifications. The logistic specification leaves $+5.7$pp unattributed; the data do not separate an $R$-confounder from other residual causes.

\emph{Measurement.} Compliance on incorrect traces ($C = 0$, 38--51\% across models) is semantically ambiguous. All-trace evaluation removes correctness selection but leaves compliance on $C = 0$ instances ambiguous (Section~\ref{sec:compliance_discussion}).

\emph{Decoding and serving deployment.} Headline runs use greedy decoding (temperature $= 0$) and English-only prompts. We have not evaluated stochastic decoding or multilingual prompts. The same model identifier behaved materially differently on two endpoints, so Table~\ref{tab:empirical} reports deployment-specific gaps.
% END sections/appendix

\end{document}